\documentclass{article}%
\usepackage[accepted]{icml2021}
\usepackage{graphicx}
\usepackage{subfigure}

\usepackage{booktabs} % for professional tables

\usepackage{hyperref}
\usepackage{enumitem}
\usepackage{amsmath}
\usepackage{amssymb}
\usepackage{amsfonts}
\usepackage{amsthm}
\usepackage{mathtools}
\usepackage[utf8]{inputenc}
\usepackage[T1]{fontenc}
\usepackage{epstopdf}

\providecommand{\U}[1]{\protect\rule{.1in}{.1in}}

\DeclareMathOperator*{\esssup}{ess\,sup}

\newtheorem{theorem}{Theorem}

\newtheorem{assumption}{Assumption}
\newtheorem{innerassumprime}{Assumption}
\newenvironment{assumprime}[1]
{\renewcommand\theinnerassumprime{#1$^\prime$}\innerassumprime}
{\endinnerassumprime}

\newtheorem{corollary}{Corollary}

\newtheorem{definition}{Definition}
\newtheorem{example}{Example}
\newtheorem{lemma}{Lemma}

\newtheorem{proposition}{Proposition}
\newtheorem{remark}{Remark}

\renewcommand{\algorithmicrequire}{\textbf{Input: }}
\renewcommand{\algorithmicensure}{\textbf{Output: }}
\icmltitlerunning{Testing Group Fairness via Optimal Transport Projections}

\newcommand{\Pnom}{\hat \PP^N} % empirical distribution
\newcommand{\QQ}{\mathbb{Q}} % use as optimization variable
\newcommand{\PP}{\mathbb{P}} % true distribution
\newcommand{\EE}{\mathbb E} % expectation operator
\newcommand{\Wass}{W}

\newcommand{\X}{\mathbb{X}}

\newcommand{\B}{\mathbb{B}}

\newcommand{\eps}{\varepsilon}
\newcommand{\st}{\mathrm{s.t.}}

\newcommand{\mc}{\mathcal}
\newcommand{\R}{\mathbb{R}}

\newcommand{\be}{\begin{equation}}
\newcommand{\ee}{\end{equation}}
\newcommand{\cov}{\mathrm{cov}}

\newcommand{\Let}{\triangleq}

\newcommand{\Inf}{\inf\limits_}
\newcommand{\opt}{^\star}
\newcommand{\ds}{\displaystyle}

\input{tcilatex}
\begin{document}
\twocolumn[
\icmltitle{Testing Group Fairness via Optimal Transport Projections}



\begin{icmlauthorlist}
\icmlauthor{Nian Si\!}{1}\hspace{-1pt}
\icmlauthor{Karthyek Murthy \!}{2}\hspace{-1pt}
\icmlauthor{Jose Blanchet\!}{1}\hspace{-1pt}
\icmlauthor{Viet Anh Nguyen\!}{1,3}\hspace{-1pt}

\end{icmlauthorlist}

\icmlaffiliation{1}{Department of Management Science \& Engineering, Stanford University}
\icmlaffiliation{2}{Engineering Systems and Design pillar, Singapore University of Technology and Design}
\icmlaffiliation{3}{VinAI Research, Vietnam}
\icmlcorrespondingauthor{Nian Si}{niansi@stanford.edu}

%\icmlaffiliation{goo}{Googol ShallowMind, New London, Michigan, USA}
%\icmlaffiliation{ed}{School of Computation, University of Edenborrow, Edenborrow, United Kingdom}
%
%\icmlcorrespondingauthor{CieuaVvvvv}{c.vvvvv@googol.com}

\icmlkeywords{Fairness, Hypothesis test, Wasserstein distance}
\vskip 0.3in
]



\printAffiliationsAndNotice{}  % leave blank if no need to mention equal contribution

\begin{abstract}
 We present a statistical testing framework to detect if a given machine learning classifier fails to satisfy a wide range of group fairness notions.
 The proposed test is a flexible, interpretable, and statistically rigorous tool for auditing whether  exhibited biases are intrinsic to the algorithm or due to the randomness in the data.
 The statistical challenges, which may arise from multiple impact criteria that define group fairness and which are discontinuous on model parameters, are conveniently tackled by projecting the empirical measure onto the set of group-fair probability models using optimal transport. This statistic is efficiently computed using linear programming and its asymptotic distribution is explicitly obtained. The proposed framework can also be used to test for testing composite fairness hypotheses and fairness with multiple sensitive attributes.
The optimal transport testing formulation improves interpretability by characterizing the minimal covariate perturbations that eliminate the bias observed in the audit.

\end{abstract}

\section{Introduction}
% \begin{itemize}
% \item[1)] Discuss the true test.

% \item[2)] Discuss technical results of discontinuous RWPI; different
%   rate with the continuous case.

% \item[3)] $\epsilon $-fairness.

% \item[4)] Suitable for handling multidimensional estimating equation:
%   the auditor simply needs to specify the transportation cost function

% \item[5)] Comparison with empirical likelihood
% \end{itemize}

%The ubiquity of algorithmic decision making has a huge impact on  daily life, including jobs, health care, financial services, etc.
Algorithmic decisions are commonly conceived to have the potential of being more objective than a human's decisions, since they are generated by logical instructions and the rules of algebra. However, recent studies indicate that this may not be the case. For example, an algorithm which helps the US criminal justice system to predict recidivism rates has been shown to falsely give a higher risk for African-Americans
than white Americans~\citep{chouldechova2017fair, ref:propublica}.
Similar biases are exhibited against female candidates in a hiring-help system developed by Amazon AI~\cite{ref:dastin2018amazon} and an ad-targeting algorithm used by Google~\cite{ref:datta2015automated}.
%Further, a hiring-help system developed by Amazon AI discriminated against female software development and technical positions~\cite{ref:dastin2018amazon}, while an ad-targeting algorithm used by Google displayed more higher-paying jobs for male than female~\cite{ref:datta2015automated}.

  A natural first explanation for the reported algorithmic biases is that the data used to train the algorithms may already be corrupted by human biases \cite{ref:buolamwini2018gender,  ref:manrai2016genetic}. Deeper inquests have revealed insights on how common learning procedures intrinsically perpetuate the biases and potentially introduce fresh ones. The usual practice of training by minimizing empirical risk, while geared towards yielding predictions that are best when averaged over the entire population, often
  %does so at the expense of minority subgroups typically
  under-represents minority subgroups in the datasets.
% There are several possible reasons for biased and unfair behaviors of algorithmic decisions. First, the data used to train the algorithm may be already corrupted by human biases \cite{ref:buolamwini2018gender,  ref:manrai2016genetic}. Second, the empirical risk minimization commonly used in machine learning algorithms may over-represent the majority and ignore minor subpopulations. Last,
Moreover, even though certain sensitive attributes are forbidden by law to be  used in the algorithm, the strong correlations between the sensitive attributes and other features potentially lead to biases in predictions.
%generate unexpected potentially unfair behaviors.
%For example, the frequency of travelling to a particular country may create an unintentional association to the race of an individual.
As reported in the studies in \cite{ref:grgic2016case, ref:garg2019counterfactual, ref:barocas2016big, ref:black2020fliptest, ref:kleinberg2018algorithmic, ref:lipton2018does}, merely masking the sensitive attributes does not address the problem.

The aforementioned biases and their %disruptive
impacts  have sparked substantial interests in the pursuit of algorithmic fairness~\cite{ref:berk2018fairness, ref:chouldechova2020snapshot, ref:corbett2017algorithmic, ref:mehrabi2019survey}. Testing whether a given machine learning algorithm is fair emerges as a question of first-order importance. In turn, designing this test for a wide range of group fairness notions (discussed in the sequel) is the main task of this paper.
%for machine learning classifiers

%This is naturally so considering the necessity of detecting biases in audits before deployment and as well from the viewpoint of serving as a conceptual building block in the larger pursuit of developing bias-free learning procedures.
%Motivated by this importance, this paper tackles the question of detecting if a given machine learning classifier fails to satisfy a generic notion of group fairness.
%that assesses the differences in the classifier's impacts on various groups.

Our proposed statistical hypothesis testing framework (testing framework for short) allows the auditors to systematically determine whether the biases exhibited in the audit procedure, if any, are intrinsic to the algorithm or due to the randomness in data. Moreover, our framework can be implemented as a black-box, without knowing the exact structure of the classification algorithm used.

% an important aim is to test if a given machine learning algorithm is fair. To properly perform the test,
% we need first find an appropriate notion of fairness.
% The task of certifying fairness needs a notion of fairness to begin with.
%A plethora of criteria for fair machine learning have been proposed in the literature, many of them are motivated by philosophical or sociological concepts or legal constraints. For example, anti-discrimination laws may prohibit making decisions based on sensitive attributes such as age, gender, race, or sexual orientation. Thus, a na\"{i}ve strategy, called fairness through unawareness, involves removing all sensitive attributes from the training data. However, this strategy seldom guarantees any fairness due to the inter-correlation issues~\cite{ref:grgic2016case, ref:garg2019counterfactual}, and thus potentially fails to generate inclusive outcomes~\cite{ref:barocas2016big, ref:black2020fliptest, ref:kleinberg2018algorithmic, ref:lipton2018does}.

For settings where sensitive attributes are not explicitly used as input to classification, fairness is measured on impact either at a group level or at an individual level~\cite{ref:barocas2016big}.
%Measuring fairness on impact, either at a group level or at an individual level~\cite{ref:dwork2012fairness}, is deemed suitable for settings where the sensitive attributes are not explicitly used.
Group fairness notions seek to measure the differences in impacts across different groups and constitute the prominent means of assessing discrimination associated with group memberships. Individual fairness, on the other hand, seeks to assess if similar users are treated similarly~\cite{ref:dwork2012fairness,ref:john2020verifying,ref:xue2020auditing}.
%The notion of counterfactual fairness  was also suggested as a measure of causal fairness~\cite{ref:garg2019counterfactual}.
Common examples of group fairness notions include disparate impact~\cite{ref:zafar2017fairness},~demographic parity (statistical parity)~\cite{ref:calders2010three},~equality of opportunity~\cite{ref:hardt2016equality}, equalized odds~\cite{ref:hardt2016equality}, etc.
The specific choice is usually driven by the philosophical, sociological or legal constraints binding the application considered.
%Common examples strive to promote \textit{group} fairness~\cite{ref:hardt2016equality}, or individual fairness~\cite{ref:dwork2012fairness}, or seek to prevent disparate treatment \cite{ref:zafar2017fairness} or avoid disparate mistreatment~\cite{ref:feldman2015certifying, ref:zafar2015fairness}.
% Working towards similar goals, notions of \textit{group} fairness focus on reducing the difference of favorable outcomes proportions among different sensitive groups. Examples of group fairness notions include disparate impact~\cite{ref:zafar2017fairness},~demographic parity (statistical parity)~\cite{ref:calders2010three, ref:dwork2012fairness},~equality of opportunity~\cite{ref:hardt2016equality}, and equalized odds~\cite{ref:hardt2016equality}. The notion of counterfactual fairness  was also suggested as a measure of causal fairness~\cite{ref:garg2019counterfactual}. %Despite the abundance of available notions,
% there is unfortunately no general consensus on the most suitable measure to serve as the industry standard. Moreover, except in trivial cases, it is not possible for a machine learning algorithm to simultaneously satisfy multiple notions of fairness \cite{ref:berk2018fairness, kleinberg2016inherent}.
%Therefore, the choice of the fairness notion is likely to remain more an art than a science.

Our testing framework applies to a generic notion of group fairness which encompasses all of the above specific group fairness notions as examples. This unifying approach can also be used in contexts requiring the use of different fairness notions simultaneously and settings with multiple groups.

Since a single fairness criterion among two groups can be reduced to testing the equality in two sample conditional means, one may consider employing a Welch's $t$-test or a permutation test. Further, a suitable adjustment of the randomness  in sample sizes, as in  ~\citet{ref:diciccio2020evaluating, ref:tramer2017fairtest} can be applied. Some other existing methods such as~\citet{ besse2018confidence} also only apply to one-dimensional criterion. Extensions to multiple impact criteria are not immediate, as is the equalized odds case criterion~\cite{ref:hardt2016equality}, or in the presence of multiple groups. Algorithmic approaches, such as in \cite{ref:saleiro2018aequitas}, lack the control of the type-I (false positive error). In contrast, the framework proposed here is applicable under general multiple impact criteria and controls the type-I error exactly.

%Since testing for group fairness is amenable to be written as a two-sample equal conditional mean test, one may consider employing a Welch's $t$-test or permutation tests in the elementary cases where only one differential impact criterion is assessed.

The statistical challenges, which may arise from the presence of multiple impact criteria, are conveniently handled in our framework by utilizing the machinery of optimal transport projections. This involves computing the test statistic by projecting, or in other words, optimally transporting the empirical measure to the set of probability models which satisfy the group fairness notion (or notions) considered. This gives a measure of plausibility of the classifier in satisfying the fairness criterion under the data-generating distribution and the fairness hypothesis is duly rejected if the test statistic exceeds a suitable threshold determined by the significance level. This threshold is determined from the limiting distribution of the test statistic obtained as one of the main results of this paper.

Performing statistical inference with a projection criterion is prevalent in statistics: \citet{ref:owen2001empirical} serves as a comprehensive reference for projections, or profile functions, that are computed based on likelihood ratio metrics or the Kullback-Liebler divergence. \citet{ref:blanchet2019robust} and \citet{ref:cisneros2020distributionally} study statistical inference with optimal transport projections. Recently, optimal transport divergences~\citep{villani} become an attractive tool in many recent machine learning studies, including missing data imputation, geodesic PCA, point embeddings, and repairing data with Wasserstein barycenters for training fair classifiers ~\citep{ref:silvia2020general,ref:del2018obtaining,zehlike2020matching}.

\citet{ref:taskesen2020statistical} uses optimal transport projections to test a smooth relaxation of the equal opportunity criterion called probabilistic fairness; see \citet{ref:pleiss2017fairness}. This relaxation is required in \citet{ref:taskesen2020statistical} to overcome the discontinuities in the classification boundaries which create technical complications when computing the optimal transport projections. Further, the resulting test statistic involves a non-convex optimization problem which is difficult to compute.
In contrast, our work resolves the technical challenges arising from the discontinuities in classification boundaries. Moreover, our test statistic is the optimal value of a linear program, whose optimal solution offers interpretability by characterizing the minimal covariate perturbations that eliminate the bias observed in the audit. We emphasize that addressing boundary discontinuities is not simply a technical improvement. As we discuss in Section~\ref{sec:null-3}, our results show different qualitative behaviors both in the scaling and the interpretation of the optimal transport projection in terms of group fairness using optimal transport projections.
In addition to enabling exact, computationally tractable, and interpretable fairness assessment for general deterministic classifiers, the technical analysis serves as a stepping stone for statistical inference in estimation tasks such as quantile regression which involve discontinuous estimating equations.

The main contributions are summarized as follows:

(1) We develop a statistical hypothesis test for assessing group fairness as per a generic notion that includes commonly used fairness criteria as special cases.
The test is computationally tractable and interpretable. Besides being applicable to settings involving multiple groups, our framework is also applicable to any
classifier algorithms, including but not limited to the logistic
regression, SVMs, kernel methods, and nearest neighbors. % and interpretable. Besides being applicable to settings involving multiple protected groups, the framework is amenable to be applied for any deterministic classifier, including but not limited to logistic regression, SVM, kernel methods and nearest neighbors.

%we propose novel hypothesis test of group fairness based on the theory of optimal transport. This statistical test works for any notion of fairness and fairness with non-binary sensitive attributes as well as multiple sensitive attributes. Further, our framework is also applicable to any classifier algorithms, including but not limited to logistic regression, SVM, kernel methods, and nearest neighbors.

%(2) We show that the test statistic is computable as the value of a linear program whose optimal solution offers interpretability by characterizing the the minimal covariate perturbations that render the classifier bias-free in the audit.

(2) We develop an extension of the statistical test for the testing problem with composite null hypothesis,  addressed here as $\epsilon$-fairness.

(3) The framework facilitates the exact use of fairness criterion, thus obviating the need to invoke relaxations in the absence of smoothness in the impact criteria defining group fairness. The resulting qualitative difference, in terms of the rate of convergence for resulting optimal transport projections, is previously unreported and could be of interest from the technical standpoint of analysing profile functions with discontinuous score functions.

%(3) Technically, we believe the statistical framework itself is of interest even beyond the scope of fairness.  Due to the nature of discontinuity  of the estimating function, the convergence rate differs from the standard results in in \citet{ref:blanchet2019robust}, \citet{ref:taskesen2020statistical} and \citet{ref:cisneros2020distributionally}; see Theorem \ref{thm:clt} and also discussions in Section~\ref{sec:null-3}. Furthermore, since the problem can be reformulated to be a linear program, we demonstrate that this framework has computational efficient solutions; see discussions in Section~\ref{sec:computation-stats}.

% ===========

% In particular, this paper is a significant improvement over~\citet{ref:taskesen2020statistical} in several aspects.
% \begin{itemize}
% \item This paper uses the original notion of fairness. And as we show in  numerical experiments, the probabilistic notion may be misleading for the true notion of fairness.
% \item \citet{ref:taskesen2020statistical} is only applicable for logistic
% regression. This paper is applicable for \text{any} classifier, as long as
% there exists an oracle for projection.
% \end{itemize}
% }

The remainder of the paper is structured as follows. In Section~\ref{sec:preliminaries}, we introduce a generic notion of group fairness and discuss the theory of optimal transport. Section~\ref{sec:simple} details the proposed statistical test for the simple fairness null hypothesis. Section~\ref{sec:composite} extends our approach to composite hypotheses. Section~\ref{sec:computation-estimation} discusses computational  methods associated with the  test. Numerical experiments presented in Section~\ref{sec:experiments} serve to demonstrate the efficacy of the test.
All technical proofs are relegated to \ref{Appenix:proof}.

%Section~\ref{sec:conclusion} concludes the paper with outlooks on the broader impact of our Wasserstein projection hypothesis testing approach.

\textbf{Notations.} We use $\|\cdot \| _{*}$ to denote the dual norm
of $\| \cdot \|$. We denote $(x)^{-}=\min \left\{ x,0\right\} $ and $
(x)^{+}=\max \{x,0\}$. We use $\Rightarrow$, $\overset{p}{\longrightarrow }$ and
$\overset{a.s.}{\longrightarrow }$ to denote convergence in distribution, in probability and convergence almost surely,
respectively. %$\mathbb{E}_{\mathbb{P}}[\cdot ]$ is the expectation under $\mathbb{P};$ we may omit the subscript $\mathbb{P}$ when the
%meaning clear according to the context.
The support of the distribution of $X$ is represented by $\text{supp}(X)$. We use $[n]$ to denote the set $\{1,2,\ldots,n\}$ and $\mathbb{R}_+^m$ to denote the positive orthant $\{x\in\mathbb{R}_+^m: x\geq 0 \}$. $\delta_{(x,a,y)}$ denotes a Dirac measure on a fixed point $(x,a,y)$.

%%%%%%%%%%%%%%%%%%%%%%%%%%%

\section{Problem Setup and Preliminaries}
\label{sec:preliminaries}

Throughout this paper we consider the classification settings in which the deterministic classifier $\mc C: \mc X \rightarrow \mc Y$ maps the input features from $\mc X \subset \mathbb{R}^d$ to output
labels in the set $\mc Y = \{0,1\}$.
Evaluation of fairness is considered with respect to a sensitive
attribute $A$ taking values in a finite set $\mathcal{A}.$ For simplicity, we consider $\mc A = \{0,1\}$ where $A = 1$ can be taken to identify the reference group. The statistical test developed in this paper
and the main results are applicable more generally to settings involving
a non-binary sensitive attribute (or) multiple sensitive attributes.
Most notions of group fairness are stated in
terms of the joint distribution $\QQ$ of $(X,A,Y),$ where $X$ is the vector of
input features, $A$ is the sensitive attribute, and $Y$ is the class
label of a random sample from the population.

\subsection{Notions of Group Fairness}
A general reference to
fairness notions can be found in \citet[Table 14]{ref:makhlouf2020on}. The statistical notion of group fairness that we consider,
encapsulated in Definition \ref{def:generic} below, is stated flexibly
to include commonly used notions such as equal opportunity
\protect\cite{ref:hardt2016equality}, predictive equality
\protect\cite{ref:corbett2017algorithmic}, equalized odds
\protect\cite{ref:hardt2016equality}, and statistical parity
\protect\cite{ref:dwork2012fairness}, etc., as special cases. This
flexibility is achieved by stating the definition in terms of a tuple
$(U,\phi),$ where $U$ is an $\R^s$-valued random vector completely dependent on $(A,Y)$ and $\phi: \R^s \times \R^s \to \R^m$ is a function
chosen to discern the differences in performance of the classifier across groups. We address $(U,\phi)$ as the
\textit{discerning tuple}.

\begin{definition}[Generic notion of group fairness]
  \label{def:generic}
  A classifier $\mc C: \mc X \to \{0,1\}$ is fair
  with respect to the discerning tuple $(U,\phi)$ under a probability
  distribution $\QQ$ if
  \be
  \label{eq:generic}
  \EE_{\QQ}[ \mc C(X) \phi(U, \EE_{\QQ}[U])] = 0.
  \ee
\end{definition}
Note that $\mathcal{C}(X) =\mathbb{I}\{\mathcal{C}(X)=1\}$, and thus equation~\eqref{eq:generic} can be seen as $\mathbb{E}_\mathbb{Q}[\mathbb{I}\{\mathcal{C}(X)=1\}\phi(U,\mathbb{E}_\mathbb{Q}[U])]=0.$  At the first glance, equation  \eqref{eq:generic} seems asymmetric as it only considers the positive prediction label $\mathcal{C}(X)=1$.
However, it is easy to check that $\mathbb{E}_\mathbb{Q} [\phi(U,\mathbb{E}_\mathbb{Q}[U])]=0$  in all the group fairness notions in Examples 1 - 6. By taking the difference, we get the symmetric guarantee that  $\mathbb{E}_\mathbb{Q}[\mathbb{I}\{\mathcal{C}(X)=0\}\phi(U,\mathbb{E}_\mathbb{Q}[U])]=0$.

Various useful notions of fairness can be obtained by varying the
choice of $(U,\phi)$ as illustrated in Examples \ref{eg:EOpp}~-~\ref{eg:Eqoppmultiple} below.
We take $\mc A = \{0,1\}$ in Examples \ref{eg:EOpp}~-~\ref{eg:SPar}.
%\begin{table}[htbp]
%\caption{Fairness notions.}
%\label{tab:fair}\centering
%\begin{tabular}{lcc}
%\toprule Notion & $U$ &$\phi(U, \EE_{\QQ}[U])]$  \\
%\midrule Equal  opportunity  &  $\mathbb{I}_{(0,1)}(A, Y))$  &  $\frac{U_1}{\EE_\QQ[U_1]}- \frac{U_2}$ \\
%Predictive  equality & 0.0945 &   0.0540   \\
%Equalized odds & 0.0895  &  0.0450   \\
%Statistical parity & 0.0900  &  0.0430   \\
%\bottomrule
%\end{tabular}%
%\end{table}

\begin{example}[Equal opportunity
  \protect\cite{ref:hardt2016equality}]
  \label{eg:EOpp}
  A classifier $\mc C: \mathcal{X} \to \{0, 1\}$ satisfies the equal
  opportunity criterion relative to a distribution
  $\mathbb{Q} $ if
  \begin{align*}
    &\mathbb{Q}\left( \mc C(X) = 1 |A=1,Y=1\right) \\
&\quad- \mathbb{Q}\left( \mc C(X) =  1 |A=0,Y=1\right) = 0.
\end{align*}
This criterion coincides with condition~\eqref{eq:generic} with the
choice $U = (\mathbb{I}_{(1,1)}(A, Y); \mathbb{I}_{(0,1)}(A, Y))$, where $\mathbb{I}_{(a,y)}(A, Y)$ denotes the indicator $\mathbb{I}(A=a, Y=y)$,  and
\begin{equation}
  \phi: (U, \EE_{\QQ}[U]) \mapsto \frac{U_1}{\EE_\QQ[U_1]}- \frac{U_2}
  {\EE_\QQ[U_2]}.
  \label{eqn:phi}
\end{equation}
\end{example}

\begin{example}[Predictive equality~\protect\cite%
  {ref:corbett2017algorithmic}]
  \label{def:PEqu}
  A classifier $\mc C:\mathcal{X} \to \{0, 1\}$ satisfies the
  predictive equality criterion relative to a distribution
  $\mathbb{Q}$ if
  \begin{align*}
    &\mathbb{Q}\left( \mc C(X) = 1 |A=0, Y = 0\right)  \\
    & \quad -\mathbb{Q}\left( \mc C(X) = 1 |A=1 , Y = 0\right) = 0.
\end{align*}
This criterion coincides with condition~\eqref{eq:generic} with the
choice $U = [\mathbb{I}_{(1,0)}(A, Y); \mathbb{I}_{(0,0)}(A, Y)]$ and $\phi$ takes the same form as the function \eqref{eqn:phi}.
\end{example}

\begin{example}[Equalized odds \protect\cite{ref:hardt2016equality}]
  \label{eg:EOdd}
  A classifier $\mc C : \mc X \to \{0, 1\}$ satisfies the equalized
  odds criterion relative to a distribution $\mathbb{Q}$ if it satisfies both equal opportunity and predictive  equality criteria.
This criterion coincides with condition~\eqref{eq:generic} with
$U = [\mathbb{I}_{(1,1)}(A, Y); \mathbb{I}_{(0,1)}(A, Y);
\mathbb{I}_{(1, 0)}(A, Y); \mathbb{I}_{(0,0)}(A,Y)]$ and
\begin{equation*}
  \phi: (U, \EE_{\QQ}[U]) \mapsto \Big[ \frac{U_1}
  {\EE_\QQ[U_1]} - \frac{U_2}{\EE_\QQ[U_2]}; \frac{U_3}{\EE_\QQ[U_3]}
  - \frac{U_4}{\EE_\QQ[U_4]} \Big].
\end{equation*}
\end{example}

\begin{example}[Statistical
  parity~\protect\cite{ref:dwork2012fairness}]
  \label{eg:SPar}
  A classifier $\mc C:\mathcal{X} \to \{0, 1\}$ satisfies the
  statistical parity criterion relative to a distribution
  $\mathbb{Q}$ if
  \begin{equation*}
    \mathbb{Q}\left( \mc C(X) = 1 |A=1\right) -
    \mathbb{Q}\left( \mc C(X) = 1 |A=0 \right) = 0.
\end{equation*}
This criterion coincides with condition~\eqref{eq:generic} with the
choice $U = [ \mathbb{I}_{1}(A); \mathbb{I}_{0}(A)]$ and $\phi$ takes the same form as the function \eqref{eqn:phi}.
\end{example}

If  the sensitive attribute takes multiple values or there are multiple sensitive attributes,   we can still define the associated fairness notions, which correspond to different choices of $(U,\phi)$.
\begin{example}[Equal opportunity with a non-binary sensitive attribute]
 Let $\mathcal{A} = \{0,1,2,\ldots,k \}$. A classifier $\mc C: \mathcal{X} \to \{0, 1\}$ satisfies the equal
  opportunity  criterion relative to a probability measure
  $\mathbb{Q} $ if
  \begin{align*}
    &\mathbb{Q}\left( \mc C(X) = 1 |A=t,Y=1\right) \\
    & \quad- \mathbb{Q}\left( \mc C(X) =  1 |A=0,Y=1\right) = 0 \quad
  \forall t\in \mathcal{A}\backslash\{0\}.
\end{align*}
This criterion coincides with condition~\eqref{eq:generic} with the
choice $U = (\mathbb{I}_{(0,1)}(A, Y);\mathbb{I}_{(1,1)}(A, Y);\ldots \mathbb{I}_{(k,1)}(A, Y))$ and $\phi=(\phi_1,\phi_2,\ldots,\phi_t)$ with
\begin{equation*}
  \phi_t: (U, \EE_{\QQ}[U]) \mapsto \frac{U_t}{\EE_\QQ[U_t]}- \frac{U_1}
  {\EE_\QQ[U_1]}  \quad
  \forall t\in \mathcal{A} \backslash\{0\}.
\end{equation*}
\label{eg:eqopp-nonbinary}
\end{example}

\begin{example}[Equal opportunity with multiple sensitive attributes]
 Suppose we have $K$ sensitive attributes, $A_1,A_2,\ldots,A_K$, all taking values in a superset $\mc A$. A classifier $\mc C: \mathcal{X} \to \{0, 1\}$ satisfies the equal
  opportunity  criterion relative to a probability measure
  $\mathbb{Q} $ if
  \begin{align*}
    &\mathbb{Q}\left( \mc C(X) = 1 |A_t=1,Y=1\right) \\
    & \quad- \mathbb{Q}\left( \mc C(X) =  1 |A_t=0,Y=1\right) = 0 \quad
  \forall t\in [K].
\end{align*}
This criterion coincides with condition~\eqref{eq:generic} with the
choice
\[
U_t = \mathbb{I}_{(1,1)}(A_t, Y) \text{ and } U_{t+K} = \mathbb{I}_{(0,1)}(A_t, Y) \quad \forall t\in [K]
\]
and $\phi=(\phi_1,\phi_2,\ldots,\phi_t)$ with
\begin{equation*}
  \phi_t: (U, \EE_{\QQ}[U]) \mapsto \frac{U_t}{\EE_\QQ[U_t]}- \frac{U_{t+K}}
  {\EE_\QQ[U_{t+K}]}  \quad \forall t\in [K].
\end{equation*}
\label{eg:Eqoppmultiple}
\end{example}
\subsection{Optimal Transport and the Wasserstein Distance}
We next introduce the notion of optimal transport costs, of which Wasserstein
distances is a special case.  Let $ P(\mc Z)$ denote the set of all
probability distributions on
$\mc Z \Let \mc X \times \mc A \times \mc Y.$

\begin{definition}[Optimal transport costs, Wasserstein distances]
Given a lower semicontinuous function $c:\mc Z \times \mc Z \rightarrow [0,\infty],$
the optimal transportation cost $W_c(\QQ_1,\QQ_2)$ between any two
distributions $\QQ_1,\QQ_2 \in  P (\mc Z)$ is given by,
\begin{align*}
  W_c(\QQ_1,\QQ_2) = \min_{\pi \in \Pi(\QQ_1,\QQ_2)} \EE_\pi \left[ {c}\left( Z,Z^\prime\right)\right],
\end{align*}
where $\Pi(\QQ_1,\QQ_2)$ is the set of all joint distributions of
$(Z, Z^\prime)$ such that the law of $Z = (X,A,Y)$ is $\QQ_1$ and that
of $Z^\prime = (X^\prime, A^\prime, Y^\prime)$ is $\QQ_2.$
\label{defn:Wass-dist}
\end{definition}
If $c(\cdot, \cdot)$ is a metric on $\mc Z$, then
$W_c(\cdot)$ is the type-1 Wasserstein distance; see \citet[Chapter 6]{villani}. The quantity $W_c(\QQ_1,\QQ_2)$ can be interpreted as
the least transportation cost incurred in transporting mass from $\QQ_1$ to $\QQ_2,$ when the cost of transporting unit
mass from location $z \in \mc Z$ to location $z^\prime \in \mc Z$ is
given by $c(z,z^\prime)$.

Throughout the paper, we assume that the function $c$ is decomposable as
\begin{align*}
&c\left((x,a,y),(x^\prime,a^\prime,y^\prime)\right)  \\
   &\quad=
  \bar{c}(x,x^{\prime})+\infty \cdot |a-a^{\prime }|+
  \infty \cdot |y-y^{\prime }|,
\end{align*}
for some $\bar{c}:\mc{X} \times \mc{X} \rightarrow [0,\infty]$ satisfying
(i) $\bar{c}(x,x^\prime) = 0$ if and only if $x = x^\prime$ and (ii)
$\bar{c}(x,x^\prime) = \bar{c}(x^\prime,x)$ for all
$x,x^\prime \in \mc X$. In the above expression, we interpret
$\infty \times 0 = 0$.
Examples of $\bar{c}(\cdot, \cdot)$ that are useful in our context include
\begin{subequations}
  \begin{equation}
    \bar{c}(x,x^\prime) = \Vert x - x^\prime \Vert,
    \label{eq:norm-c}
  \end{equation}
  and also
  \begin{equation}
    \bar{c}(x,x^\prime) = k(x,x) -2k(x,x') + k(x',x'),
        \label{eq:kernel-c}
  \end{equation}
\end{subequations}
where $k: \mc X \times \mc X \rightarrow \mathbb{R}$ is a suitable
reproducing kernel. Another useful example of $\bar{c}(\cdot)$ is specified
in terms of the discrete metric suitable for use in the presence of
discrete categorical features: Suppose that the feature vector
$X = (X_D,X_C),$ with $X_D$ denoting the set of discrete features
taking values in a countable set $\mathbb{D} \subset \mathbb{R}^{d_1}$
and $X_C$ denoting the set of continuous features taking values in
$\mathbb{R}^{d_2}.$ We have $d_1 + d_2 = d.$ In this instance, it is
feasible to restrict the transportation to elements in
$\mathbb{D} \times \mathbb{R}^{d_2}$ by considering
\begin{align}
  &\bar{c}\left((x_D,x_C), (x_D^\prime,x_C^\prime)\right) \notag\\
  &\quad = \Vert x_C - x_C^\prime \Vert + \delta \mathbb{I}\{\{x_D, x'_D\} \subset \mathbb{D},x_D\neq x'_D  \} \notag \\
  &\quad+ \infty \cdot \mathbb{I}\{\{x_D, x'_D\} \nsubseteq \mathbb{D},x_D\neq x'_D \},
  \label{discrete-c}
\end{align}
for some $\delta > 0$.
Further, we allow the cost function to be dependent on the sensitive attribute. Following the same line,~\citet{yue2021} recently demonstrates a test power gain by tuning properly a sensitive-attribute-dependent transportation cost function.
%%%%%%%%%%%%%%%%%%%%%%%%%%%%%%%%%%%

\section{Test For Simple Null Hypothesis via Optimal Transport}
\label{sec:simple}
\noindent
Recall that $ P(\mc Z)$ denotes the set of all
probability distributions on
$\mc Z \Let \mc X \times \mc A \times \mc Y.$ Let
\begin{equation*}
  \mc F = \left\{ \QQ \in  P(\mc Z): \EE_{\QQ}[ \mc C(X) \phi(U, \EE_\QQ[U]) ] = 0\right\}
\end{equation*}
be the collection of distributions under which the classifier
$\mc C(\cdot)$ is fair, as deemed by Definition
\ref{def:generic}. Given $N$ independent samples
$\{ x_{i},a_{i},y_{i}\}_{i=1}^{N}$ from a distribution $\PP$ of
$(X,A,Y),$ we are interested in the statistical test with the
hypotheses
\[
  \mathcal{H}_{0}: \PP \in \mc F \quad \text{against} \quad
  \mathcal{H}_{1}: \PP \not\in \mc F.
\]
With the null hypothesis $\mc H_0$ being that the classifier
$\mc C(\cdot)$ is fair, our statistical test will detect the failure of $\mc C(\cdot)$ in meeting the fairness criterion (in Definition \ref{def:generic}) under the data generating distribution. To develop a suitable test statistic, let
$\Pnom= N^{-1}\sum_{i=1}^{N}\delta _{\left( x_{i},a_{i},y_{i}\right)
}$ denote the empirical measure of the samples obtained from a distribution $\mathbb{P} \in P (\mc Z)$. We define the projection of $\Pnom$ onto $\mc F$ as
\begin{align}
\mc P(\Pnom) &\Let \Inf{\QQ \in \mc F}~\Wass_{c}(\QQ, \Pnom) \notag \\
&= \left\{
\begin{array}{cl}
\inf & \Wass_c(\QQ, \Pnom)  \\
\st & \EE_{\QQ}[ \mc C(X) \phi(U, \EE_\QQ[U]) ] = 0.
\end{array}
\right.
\label{problem:primal}
\end{align}

We adopt the statistical hypothesis framework: for a prespecified significance level $\alpha$,
\begin{center}
  reject $\mc H_0$ if $s_{_N} > \eta_{1-\alpha}$,
\end{center}
where $s_{_N}$ is a test statistic that depends on the projection distance $\mc P(\Pnom)$, and $\eta_{1-\alpha}$ is the $(1-\alpha)\times 100\%$ quantile of a limiting distribution.

%The rest of this section unfolds as follows: Section~\ref{sec:null-1} details a tractable reformulation of $\mc P(\Pnom)$.  Section~\ref{sec:null-2} provides a detailed analysis of the test: Theorem~\ref{thm:clt} essentially leads to the statistic $s_N = N \times \mc P(\Pnom)$ and the threshold $\eta_{1-\alpha}$ can be chosen as the quantile of a generalized chi-squared distribution. In Section~\ref{sec:null-3}, we study qualitative structures of the Wasserestein projection report some insights on the convergence rate of the central limit theorem.

\subsection{Linear Programming Formulation for Projection}
\label{sec:null-1}

Our aim here is to reformulate the \textit{in}finite dimensional projection formulation~\eqref{problem:primal} as a finite dimensional linear program. For this purpose, let us define
$d:\mc X \rightarrow [0,\infty]$ as
\begin{align}
  d(x) \Let \inf \left\{\bar{c}(x,x^\prime): x^\prime \in \mc X, \ \mc C(x^\prime) = 1 - \mc C(x)\right\},
  \label{eq:defn-d}
\end{align}
which gives a measure of distance to the region with classifier label
different from that at $x,$ and $d(x)=0$ means that $x$ on the decision boundary.  The value $d(x)$ is readily computed for commonly used classifiers such as linear classifiers (as shown in the proof of Lemma 1 below in the supplement) and kernelized classifiers. In the case of a classifier defined in terms of kernels, say as in, \[\mathcal{C}(x) = \mathbb{I}\left(\sum_{i=1}^n \alpha_i k(x_i,x) + b \geq 0\right),\] one may use the transportation cost \eqref{eq:kernel-c}  and $d(x)$ admits a closed-form expression \[d(x) = \left(\sum_{i=1}^n \alpha_i k(x_i,x) + b\right)^2/(\alpha^\top \mathbb{K} \alpha),\] where $\mathbb{K}$ is an $n \times n$  matrix with entries $\mathbb{K}_{i,j} = k(x_i,x_j).$

\begin{proposition}[Primal reformulation]
  \label{prop:primal} The projection distance $\mc P(\Pnom)$ is equal
  to the optimal value of a linear program. More specifically, we have
\begin{equation}
\label{problem:reformulate}
\mc P(\Pnom) = \left\{
\begin{array}{cl}
  \min_p & \ds \frac{1}{N}\sum_{i \in [N]} p_i d(x_i) \\
  \st & p \in [0, 1]^N, \\
       & \ds \sum_{i \in [N]} [1 - 2 \mc C(x_i)] \phi(u_i, \EE_{\Pnom}[U]) p_i \\
       &\quad = -
         \ds\sum_{i\in [N]}\mc C(x_i) \phi(u_i, \EE_{\Pnom}[U]). %
\end{array}
\right.
\end{equation}
\end{proposition}

Naturally, one may study the above linear program by considering its
dual formulation. Define the following function
\begin{align*}
& \mc D(\Pnom)    \Let \\
&\max_{\gamma \in \mathbb{R}^{m}}\left\{
\begin{array}{l}
\frac{1}{N}\sum_{i \in [N]} \gamma^\top \phi(u_i, \EE_{\Pnom}[U])\mc C(x_i) + \\ \left(d(x_i) + \left[1-2\mc C(x_i) \right]\gamma^\top \phi(u_i, \EE _{\Pnom}[U]) \right)^{-}
\end{array}
    \right\},
\end{align*}
where recall the notation that $(x)^{-} = \min\{x,0\}.$ Strong duality
of linear programming asserts that $\mc P(\Pnom)$ and $\mc D(\Pnom)$ are dual to each other.
%Please refer the proof of Proposition \ref{prop:duality} for arriving at the above expression of $\mc D(\Pnom).$
\begin{proposition}[Strong duality]
  \label{prop:duality}
  Strong duality holds, i.e.,
  $\mc P(\Pnom) = \mc D(\Pnom)$.
\end{proposition}

\subsection{Asymptotic Behavior of the Projection Distance}
\label{sec:null-2}

The goal of this subsection is to study the limiting behavior of the projection distance $\mc P(\Pnom)$
as the sample size $N$ increases. Proposition~\ref{prop:duality} implies that it is sufficient to examine the asymptotic behavior of $\mc D(\Pnom)$. To present the
regularity assumptions under which the limiting behavior can be
unravelled, we set $\mu = \EE_{\PP}[U]$ and define $\Phi$ as
\begin{align*}
  \Phi: \mc X \rightarrow \mathbb{R},\quad \Phi(X) = (2\mc C(X) - 1)d(X).
\end{align*}

% The expression for $\mc D(\Pnom)$ is amenable for investigating the
% limiting behavior of the projection distance $\mc P(\Pnom).$

\begin{assumption}[Continuous density and derivatives]
  There exists $\eta > 0$ such that the below
    conditions are satisfied:
    \begin{enumerate}[label=\emph{\alph*})]
    \item The probability distribution of $\Phi(X)$ has a positive
      continuous density $f(\cdot)$ in the interval $(-v, v),$
      i.e., $\PP(\Phi(X) \in [-v,u)) = \int_{-v}^u f(\nu)\mathrm{d}\nu$ for
      $u \in (-v, v).$ \label{assump:1.1}
    \item For every $u \in \text{supp}(U),$ the function
      $\phi(u,z)$ has a continuous derivative (Jacobian matrix) $\phi_z(u,z)$ in the
      neighborhood $z$ satisfying $\Vert z - \mu \Vert_2 < v.$ In
      addition,
      $\Sigma_1 \Let \mathbb{E}_{\PP}[ \phi ( U,\mu) \phi ( U,\mu) ^{\top
      } \vert\, d(X) = 0 ] \succ 0.$ \label{assump:1.2}
    \end{enumerate}
  \label{assumption}

\end{assumption}
% The following additional assumption is sufficient if the support of $U$ has finite cardinality, as is the case with the fairness notions considered in Examples \ref{eg:EOpp}~-~\ref{eg:Eqoppmultiple}.
\begin{assumption}[Continuous conditional probability]
  For the case where $\text{supp}(U)$ is a finite set, the
    conditional probability $\PP(U=u\,\vert\,\Phi(X) = t)$ is continuous
    around $t = 0$ for every $u \in \text{supp}(U).$
    \label{assump:finite}
\end{assumption}

We are now ready to state the main result of this section concerning the asymptotic behavior of the projection distance.

\begin{theorem}[Limit theorem for $\mc D (\Pnom)$]
  \label{thm:clt}
  Suppose that
  \sloppy{$\left\{ X_{1},U_{1}\right\} ,...,\{X_{n},U_{n}\}$} are
  independently obtained from the distribution $\PP$ and that
  Assumptions \textnormal{\ref{assumption}} and
  \textnormal{\ref{assump:finite}} are satisfied.  Then under the
  null hypothesis $\mc H_0$,
\begin{equation*}
N\times\mathcal{D}( \Pnom) \Rightarrow \max_{\gamma \in \mathbb{R}%
^{m}}\left\{ \gamma^{\top } V-\frac{1}{2}\gamma ^{\top }S\gamma \right\} =%
\frac{1}{2}V^\top S^{-1}V,
\end{equation*}
where $S = f(0)\Sigma_1$, $V \sim \mc N (0,\Sigma)$,  $\Sigma$ is
the covariance matrix of
$\phi (U,\mu ) \mc C(X) + \mathbb{E}_{\mathbb{P}}\left[ \phi
  _{z}(U,\mu )\mc C(X) \right] U$, and $\Rightarrow$ denotes the  convergence in distribution.
\end{theorem}

The finite cardinality of the outcome space of $U$ in~Assumption~\ref{assump:finite} is not restrictive: Theorem~\ref{thm:clt} still holds under \textit{in}finite cardinality under an equivalent assumption. Details can be found in \ref{Appenix:proof}. For the fairness notions in Examples \ref{eg:EOpp}~-~\ref{eg:SPar}, we report in Corollary \ref{cor:fairness-lim} below the specific closed-form limit distributions obtained from Theorem \ref{thm:clt}.

\begin{corollary}
\label{cor:fairness-lim}
Suppose that $\phi(U, \mu )=\frac{U_1}{\mu_1} -
  \frac{U_2}{\mu_2}$ with $U_1,U_2$ satisfying $U_1U_2 = 0$ (with probability 1), as in Examples \ref{eg:EOpp}, \ref{def:PEqu}, and \ref{eg:SPar}. Then we have the limiting distribution,
  \begin{align*}
 &  V^\top S^{-1}V/2 \\
 =&\frac{\sigma^2\chi^2(1)}{2f(0)\left({\mu_2^2\EE_\PP[U_1|d(X)=0]}+\mu_1^2{\EE_\PP[U_2|d(X)=0]}\right)} ,
  \end{align*}
  where $\chi^2(1)$ is a chi-squared distribution with one degree of freedom and
  \begin{align*}
  \sigma^2 =  &\mathrm{var}\left\{ \mc C(X) \left(\mu_2U_1-\mu_1U_2 \right) \right.\\
  + & \left. U_2 \EE_\PP\left[U_1 \mc C(X)\right]  -  U_1 \EE_\PP\left[U_2 \mc C(X)\right] \right\}.
  \end{align*}
  For Example \ref{eg:EOdd}, we have
  \begin{align*}
  &V^\top S^{-1}V/2  \\
  = & \frac{\sigma_1^2 \chi_1(1)^2 }{2f(0)\left({\mu_2^2\EE_\PP[U_1|d(X)=0]}+\mu_1^2{\EE_\PP[U_2|d(X)=0]}\right)}\\
   +  &\frac{\sigma_2^2 \chi_2(1)^2}{2f(0) \left({\mu_4^2\EE_\PP[U_3|d(X)=0]}+\mu_3^2{\EE_\PP[U_4|d(X)=0]}\right)}.
  \end{align*}
  where $\chi_1^2(1)$ and $\chi_2^2(1)$  are two independent chi-squared distributions with one degree of freedom and
  \begin{align*}
  \sigma_1^2 =  &\mathrm{var}\left\{ \mc C(X) \left(\mu_2U_1-\mu_1U_2 \right) \right.\\
  + & \left. U_2 \EE_\PP\left[U_1 \mc C(X)\right]  -  U_1 \EE_\PP\left[U_2 \mc C(X)\right] \right\}, \\
  \sigma_2^2 =  &\mathrm{var}\left\{ \mc C(X) \left(\mu_4U_3-\mu_3U_4 \right) \right.\\
  + & \left. U_4 \EE_\PP\left[U_3 \mc C(X)\right]  -  U_3 \EE_\PP\left[U_4 \mc C(X)\right] \right\}.
  \end{align*}
\end{corollary}
%The proof of Theorem~\ref{thm:clt} is in \ref{prf:sec_simple}.
 Assumption~\ref{assumption}\ref{assump:1.1} is satisfied for a broad
class of classification models interesting in practice. Lemma
\ref{lma:sufficient_assump_1.1} below identifies that
Assumption~\ref{assumption}\ref{assump:1.1} is satisfied even if
there are some discrete features.

\begin{lemma}
  Suppose that $X=(X_D,X_C)$, where $X_D$ takes values in a finite
  subset $\mathbb{D} \subset \mathbb{R}^{d_1}, d = d_1 + d_2$ and
  $X_C$ is an $\mathbb{R}^{d_2}$-valued random vector whose
  conditional distribution $X_C\,|\,X_D=x$ has a positive density in
  $\mathbb{R}^{d_2}$ for every $x \in \mathbb{D}$.
%We further assume that there exists $0<\underline(M)<\overline{M}<+\infty$ the cost function satisfies
%\begin{align*}
%c(x,x') &\geq \underline{M} \|x-x\| \text{ and } \\
%c(x,x') &\leq \overline{M} \|x_C-x_C'\| \text{ when } x_D=x_D'.
% \end{align*}
  Let the cost function $\bar{c}(\cdot)$ be given by (\ref{eq:norm-c})
  or (\ref{discrete-c}).  Further, if the classifier is written as
  $\mc C(X) = \mathbb{I}\{\ell(\theta^\top X) \geq \tau\}$ where
  $\ell(\cdot)$ is a continuous and increasing function with
  $\lim_{x \rightarrow +\infty}\ell(x)>\tau>\lim_{x \rightarrow
    -\infty} \ell(x)$ and $%
  \theta = (\theta_D,\theta_C)$ with $\theta_C\neq0,$ we have that
  Assumption~\ref{assumption}.\ref{assump:1.1} is satisfied.
\label{lma:sufficient_assump_1.1}
\end{lemma}

With $\mc P(\Pnom) = \mc D(\Pnom)$ as in Proposition
\ref{prop:duality}, Theorem \ref{thm:clt} reveals that one can use
$s_{_N} \Let N \times \mc P(\Pnom)$ as a test statistic to reject
$\mc H_0.$ In particular, for a prespecified significance level
$\alpha \in (0,1), $ let $\eta_{1-\alpha}$ denote the
$(1-\alpha)\times 100\%$ quantile of the generalized chi-squared %the word chi should not be capitalized. In the Greek alphabet, the capital letter for $\chi$ is X Thanks!
distribution given by the law of $V^\top S^{-1} V/2.$ Specific computation and estimation procedures required to compute the
statistic $s_{_N}$ and the quantile $\eta_{1-\alpha}$ are discussed in
Section~\ref{sec:computation-estimation}.

\subsection{The Structure of the Wasserstein Projection}
\label{sec:null-3}
In this subsection, we characterize the projection measure $\QQ$ and provide a heuristic justification for the convergence rate of Theorem \ref{thm:clt}.
Note that the proof of Proposition \ref{prop:primal} also leads to an
$\eps$-optimizer sequence for (\ref{problem:primal}).
\begin{proposition}[$\eps$-optimizer]
\label{prop:eps_opt}
Suppose $d(x_i)<+\infty$ for every $i \in [N]$. Let ${x}^\varepsilon_i$
be an $\varepsilon$-optimizer obtained by solving (\ref{eq:defn-d}) with $x = x_i,
$ for $i \in [N].$ Let $\{p_i\opt%
\}_{i=1}^N$ be an optimal solution of the problem %
\eqref{problem:reformulate}. Then, the measure
\begin{equation*}
\mathbb{Q}^\eps \Let\frac{1}{N}\sum_{i=1}^{N}(1-p_i\opt)\delta _{\left(
x_{i},a_{i},y_{i}\right) } + p_i\opt\delta _{\left( {x}%
^\varepsilon_i,a_{i},y_{i}\right) }
\end{equation*}
is an $\varepsilon$-optimizer of the problem \eqref{problem:primal}.
\end{proposition}
Proposition \ref{prop:eps_opt} indicates in the optimal transportation plan,
the transporter moves mass from $x_i$ to $x_i^\eps$ when $p_i\opt\neq0$.
 Since the optimal solution of a linear programming problem occurs at corner points,
 most of $p_i\opt$ should be either zero or one.  Further, the proof of Theorem \ref{thm:clt} shows that the number of non-zero values in $\{p_i\opt%
\}_{i=1}^N$ is of the order $O_p(N^{1/2})$ and for each non-zero $p_i\opt$,
the moving distance $d(x_i)$ is of the order $O_p(N^{-1/2})$ under the null hypothesis.
Therefore, $\mc D(\Pnom)$ is of the order $O_p(N^{-1})$.
This statistical phenomenon is due to the discontinuity
of the estimating function $\mc C(X) \phi(U, \EE_\QQ[U]$) in $X$,
where the transporter is able to move a small amount of probability mass,
 but the move results in a significant change of the value for the estimating
function around the discontinuity region. The $O_p(N^{-1})$ convergence rate
is in contrast to the rate in \citet{ref:blanchet2019robust}, \citet{ref:taskesen2020statistical} and \citet{ref:cisneros2020distributionally},
where the estimating function is assumed to be continuous.
For the continuous estimating function, it is optimal to move every point $O_p(N^{-1/2})$ distances,
which results in a $O_p(N^{-1/2})$ convergence rate. Therefore, let us emphasize again the key qualitative difference between our contributions and those of \citet{ref:taskesen2020statistical}.
A statistical noise gives the empirical appearance of unfairness in two ways: (A) small statistical fluctuations around all data points; (B) a small sub-population with large outcome fluctuations around the decision boundary. \citet{ref:taskesen2020statistical} studied scenario (A) and our paper studies scenario (B).
%%%%%%%%%%%%%%%%%%%%%%%%%%%%%%%%%%%

\section{Test For Composite Null Hypothesis via Optimal Transport}
\label{sec:composite}
In settings where the notion of exact group fairness becomes restrictive or unattainable, it becomes attractive to test whether the deviation from fairness, if any, from a given group's viewpoint is not more than a prespecified small extent.
The question of verifying $\epsilon$-fairness, from an one-sided perspective, can
be similarly formulated as follows. Following Section~\ref{sec:simple}, we define
\begin{equation*}
\mc F_\epsilon = \left\{ \QQ: \EE_{\QQ}[ \mc C(X) \phi(U, \EE_\QQ[U]) ] \le
\epsilon \right\},
\end{equation*}
where $\epsilon \in \mathbb{R}^m_+$ is a tolerance level prespecified
by the fairness auditor. We are interested in the statistical test
with the hypotheses
\begin{equation}
\mathcal{H}_{0}: \PP \in \mc F_\epsilon \quad \text{against} \quad \mathcal{H%
}_{1}: \PP \not\in \mc F_\epsilon.
\label{eq:comp-H0}
\end{equation}
A suitable statistical test for this formulation will serve the
purpose of detecting failure of $\mc C(\cdot)$ in meeting the
one-sided fairness condition within an $\epsilon$ tolerance.
%\subsection{Linear programming formulation for projection onto $\mc F_\epsilon$}
Similar to Problem \ref{problem:primal}, we define the projection of $\Pnom$ onto $\mc F_\epsilon$ as
\begin{align*}
\mc P_\epsilon(\Pnom)
%& \Let \Inf{\QQ \in \mc F_\epsilon}~\Wass(\QQ, \Pnom)\\
&\Let
\left\{
\begin{array}{cl}
\inf & \Wass_c(\QQ, \Pnom) \\
\st & \EE_{\QQ}[ \mc C(X) \phi(U, \EE_\QQ[U]) ] \leq \epsilon .
\end{array}
\right.
\end{align*}

\subsection{Linear Programming Formulation for Projection}
\label{sec:composite-1}
\begin{proposition}[Primal reformulation]
  \label{prop:primal_eps} The projection distance
  $\mc P_\epsilon(\Pnom)$ is equal to the optimal value of a linear
  program. More specifically, we have
\begin{align}
&\mc P_\epsilon(\Pnom) \label{problem:reformulate_eps} \\
&= \left\{
\begin{array}{cl}
\min_p & \ds \frac{1}{N}\sum_{i \in [N]} p_i d(x_i)  \\
\st & p \in [0, 1]^N \\
& \ds\sum_{i \in [N]} (1 - 2 \mc C(x_i)) \phi(u_i, \EE_{%
\Pnom}[U]) p_i   \\
& \qquad + \ds\sum_{i\in [N]}\mc C(x_i) \phi(u_i, \EE_{\Pnom}[U])
 \leq N\epsilon.\notag
\end{array}
\right.
\end{align}
\end{proposition}
As in Section~\ref{sec:simple}, $\mc P_\epsilon(\Pnom)$ is amenable to
be studied via the respective dual function,
\begin{align*}
&\mc D_\epsilon(\Pnom) \Let \\
& \max_{\gamma \in \mathbb{R}_+^{m}}~-\gamma^\top
\epsilon +  \frac{1}{N}\sum_{i \in [N]} \Big\{\gamma^\top
\phi(u_i, \EE_{\Pnom}[U])\mc C(x_i) \\
& \qquad + \left(d(x_i) + \left(1-2\mc %
C(x_i) \right)\gamma^\top \phi(u_i, \EE_{\Pnom}[U]) \right)^{-}\Big\}.
\end{align*}

\begin{proposition}[Strong duality]
  \label{prop:duality2}
 Strong duality holds, i.e.,
  $\mc P_\epsilon(\Pnom) = \mc D_\epsilon(\Pnom)$.
\end{proposition}

\subsection{Asymptotic Behavior of the Projection Distance}
\label{sec:composite-2}

We next study the
limit of $\mc D_{\epsilon }(\Pnom)$ as the sample size $N$ tends to
infinity. In order to state the theorem, let us introduce notation for
asymptotic stochastic ordering. We say that a sequence of random
elements $\{A_n\}_{n \geq 1}$ satisfies $A_{n}\lesssim _{D}B$ if for
every continuous and bounded non-decreasing function $g$,
\begin{equation*}
 \limsup_{n\rightarrow \infty }\mathbb{E}\left[ g(A_{n})\right]
  \leq\mathbb{E}\left[ g(B)\right] .
\end{equation*}

\begin{theorem}
  \label{thm:clt2}
  Suppose that
  \sloppy{$\left\{ X_{1},U_{1}\right\} ,...,\{X_{n},U_{n}\}$} are
  independently obtained from the distribution $\PP$ and that
  Assumptions \textnormal{\ref{assumption}} and
  \textnormal{\ref{assump:finite}} are satisfied.  Then under the
  null hypothesis $\mc H_0$,
  \begin{equation}
    N\times \mathcal{D}_{\epsilon }(\mathbb{\hat{P}}^{N})\lesssim
    _{D}\max_{\gamma \in \R_{+}^{m}}\left\{ \gamma ^{\top }V-\frac{1}{2}\gamma
      ^{\top }S\gamma \right\} ,
    \label{eq:asymp-bound}
\end{equation}%
where $S = f(0)\Sigma_1,$ $V \sim \mc N (0,\Sigma),$ and $\Sigma$ is
the covariance of
$\phi (U,\mu ) \mc C(X) + \mathbb{E}_{\mathbb{P}}\left[ \phi
  _{z}(U,\mu )\mc C(X) \right] U.$
  In particular, if $\phi (U,\mu )$ is one-dimensional ($m=1$), we have
  \[
  \max_{\gamma \in \R_{+}^{m}}\left\{ \gamma ^{\top }V-\frac{1}{2}\gamma
      ^{\top }S\gamma \right\}=\frac{1}{2}S^{-1}V^2 \mathbb{I}\{V \geq 0\}.
  \]
\end{theorem}

With $\mc P_\epsilon(\Pnom) = \mc D_\epsilon(\Pnom)$ as in Proposition
\ref{prop:duality2},
Theorem \ref{thm:clt2} reveals that one can use
$s_{_N}(\epsilon) \Let N \times \mc P_\epsilon(\Pnom)$ as a test
statistic to reject $\mc H_0$ and $\eta_{1-\alpha}$, defined by
 the $(1-\alpha)\times 100\%$ quantile of the right hand side bounding variable in (\ref{eq:asymp-bound}), as a threshold. We then follow the same hypothesis testing procedure defined in Section~\ref{sec:simple}. Since Theorem \ref{thm:clt2} only provides a stochastic upper-bound, we actually use a \textit{conservative} quantile and the type I error is less than or equal to the desired significance level $\alpha$ asymptotically.

%%%%%%%%%%%%%%%%%%%%%%%%%%
\section{Computation and Estimation Procedures}
\label{sec:computation-estimation}
%\subsection{Computations of the statistics $s_{_N}$ and $s_{_N}(\epsilon)$}
\subsection{Computations of the Test Statistic}
\label{sec:computation-stats}
Based on Proposition \ref{prop:primal}, we propose a
sorting-based algorithm for computing $\mc P(\Pnom)$ for one-dimensional
$\phi(\cdot)$  (that is,  $m=1$). The steps involve transporting  points which are close to the decision boundary and have significant contributions towards improving fairness if prediction labels are flipped. The exact steps are described in Algorithm \ref{alg:sort}. Note that the algorithm requires only information on  $\{\mathcal{C}(x_i), d(x_i): i \in [N]\}$, instead of the whole functional structure of the classifier $\mathcal{C}$.
\begin{algorithm}[H]

\caption{Computing $\mc P (\Pnom)$ for one-dimensional
$\phi(\cdot)$}
\label{alg:sort}
 \begin{algorithmic}[1]
   \STATE \algorithmicrequire Data
   $\{d(x_i),\mc C(x_i),\phi(u_i,\EE_{\Pnom}[U])\}_{i=1}^N$.

   \STATE \algorithmicensure the optimal value
   $\mc P(\Pnom)$.

   \STATE Let
   $s\gets -\sum_{i\in [N]}\mc C(x_i) \phi(u_i, \EE_{\Pnom}[U])$;

   \STATE For $i \in [N]$, compute
   \[t_i \gets d(x_i)^{-1}(1-2 \mc C(x_i)) \phi(u_i,
   \EE_{\Pnom}[U]){\rm sgn}(s);\]

   \STATE Sort
   $t_1,\ldots,t_N$ in descending order, where $t_{(i)}$ denotes the
   $i$-th largest one and let $d_{(i)}$ be the corresponding distance;

   \STATE Initialize $V = 0$ and let $s \gets |s|$;

   \FOR{$i\gets1$ to $T$ }

   \IF{$t_{(i)}d_{(i)} < s$}
   \STATE $s\gets s-t_{(i)}d_{(i)}$ and $V\gets V+d_{(i)} $;
   \ELSE \STATE $V\gets V+t_{(i)}^{-1}s $ and break;
   \ENDIF

   \ENDFOR
   \STATE Output $\mc P(\Pnom) \gets V/N$.
   \end{algorithmic}

 \end{algorithm}
\vspace{-4mm}
 With the output $\mc P(\Pnom)$ returned by Algorithm
 \ref{alg:sort}, computation of the test statistic
 $s_{_N} = N\times\mc P(\Pnom)$ is immediate. To obtain
 $\mc P_\epsilon(\Pnom)$ similarly, one may modify Line 3 in
 Algorithm \ref{alg:sort} as in
  \[
   s \gets
   - \left(\sum_{i\in [N]}\mc C(x_i) \phi(u_i, \EE_{\Pnom}[U]) - N\epsilon \right)^+.
 \]
 With ${\rm sgn}(0) = 0$ assigning $t_i = 0$ for all $i \in [N]$ in
 Step 4, we take $0/0 = 0$ in Line 11 in Algorithm \ref{alg:sort}
 in order to obtain $\mc P_\epsilon(\Pnom).$ It is easy to see that the time complexity of Algorithm \ref{alg:sort} is the same of the time complexity of the sorting algorithm, which is generally $O(N\log N)$. For instances where
 $m > 1$, one may solve either problem \eqref{problem:reformulate} or problem
 \eqref{problem:reformulate_eps} with a standard linear program solver
 to obtain the respective values $\mc P(\Pnom)$ or
 $\mc P_\epsilon(\Pnom)$, which is also solvable in polynomial time. Therefore, our hypothesis test is more computationally efficient than the test proposed in~\citet{ref:taskesen2020statistical}, which requires solving a non-convex optimization problem.

\subsection{Computations of the Quantile of the Limiting Distributions}
\label{sec:kernel_estimate}
We use the conditional density estimator and the Nadaraya-Watson estimator \citep[%
Section 1]{ref:tsybakov2008introduction} to estimate $f(0)$ and $\Sigma_1$, i.e.,
\begin{align*}
&\hat{f}(0)=\frac{1}{Nh}\sum_{i\in [N]}K\left(\frac{\Phi(x_i)}{h}\right), \text{ and }\\
& \hat{\Sigma}_1 =
\frac{\ds\sum_{i\in [N]} K\left(\frac{\Phi(x_i)}{h}\right)\phi \left(
u_i,\EE_{%
\Pnom}[U]\right) \phi \left( u_i,\EE_{%
\Pnom}[U]\right) ^{\top }}{\ds\sum_{i\in [N]}K\left(\frac{\Phi(x_i)}{h}\right)},
\end{align*}
where $h>0$ is the bandwidth parameter, and $K(\cdot)$ is a kernel function
that is symmetric and integrates to one. By combining the above two estimates, an empirical estimate for $S$, denoted by $\hat{S}$ is computed via,
\begin{equation*}
\frac{1}{Nh}\ds\sum_{i\in [N]} K\left(\frac{\Phi(x_i)}{h}\right)\phi \left(
u_i,\EE_{%
\Pnom}[U]\right) \phi \left( u_i,\EE_{%
\Pnom}[U]\right) ^{\top }.
\end{equation*}
Under some mild conditions, by choosing $h=O(N^{-1/5})$, we have $\|S-\hat{S}%
\|=O(N^{-2/5})$; see, for example, \citet[Theorem 4.2.1]%
{ref:hardle1990applied} and \citet[Proposition 1.7]%
{ref:tsybakov2008introduction}. By combining the empirical covariance estimator for $\Sigma$, we  obtain a  quantile estimate  $\hat{\eta}_{1-\alpha}$.

\section{Numerical Experiments}
\label{sec:experiments}
\label{sec:real}
Our experiments use the following three datasets: Arrhythmia~\citep{ref:Dua:2019}, COMPAS~\citep{ref:propublica}
 and Drug \cite{fehrman2017five}. The details of the datasets are provided in \ref{sec:dataset}.

In the first experiment, we test the fairness of the
Tikhonov-regularized logistic and SVM classifiers by the equal opportunity criterion. We randomly split 70\%-30\% of the data as a train-test set.  Figure~\ref{fig:regularization} reports the test statistics, fairness rejection threshold, and the accuracy of the classifier. Figure~\ref{fig:compas} shows the result of regularized logistics classifier in COMPAS dataset, while Figure \ref{fig:drug} shows the result of SVM classifier in the Drug dataset. We observe that a strong regularization only reduces the test statistics very mildly, and the Wasserstein projection tests suggest we reject the fair null hypothesis even when the regularization power is sufficiently large, which presents a different phenomenon from the probabilistic fairness test results shown in \citet{ref:taskesen2020statistical}. We here provide a heuristic explanation for this difference. Consider a logistic classifier $ \mc C(x) = \mathbb{I}\{1/ (1+\exp(-\theta ^\top x)) \geq 0.5\}$. Since regularization usually induces shrinkage, the regularized classifier could be approximated by  $ \mc C_\varepsilon(x)=\mathbb{I}\{ 1/ (1+\exp(-\varepsilon \theta ^\top x)) \geq 0.5\}$, and large regularization power corresponds to small $\varepsilon$. Note that $\mc C_\varepsilon(x)=\mc C(x)$ no matter how small $\varepsilon >0$ is. However, for the probabilistic notion, $\mc C_\varepsilon$ will output approximately equal  probabilities for both labels, which tends to be probabilistic fair when $\varepsilon$ is very small. The experiment thus demonstrates probabilistic fairness does not imply the exact fairness in general.

\begin{figure}[thbp]
\centering
\subfigure[Regularized logistics classifier in COMPAS]{
\label{fig:compas} \includegraphics[width=3.5in]{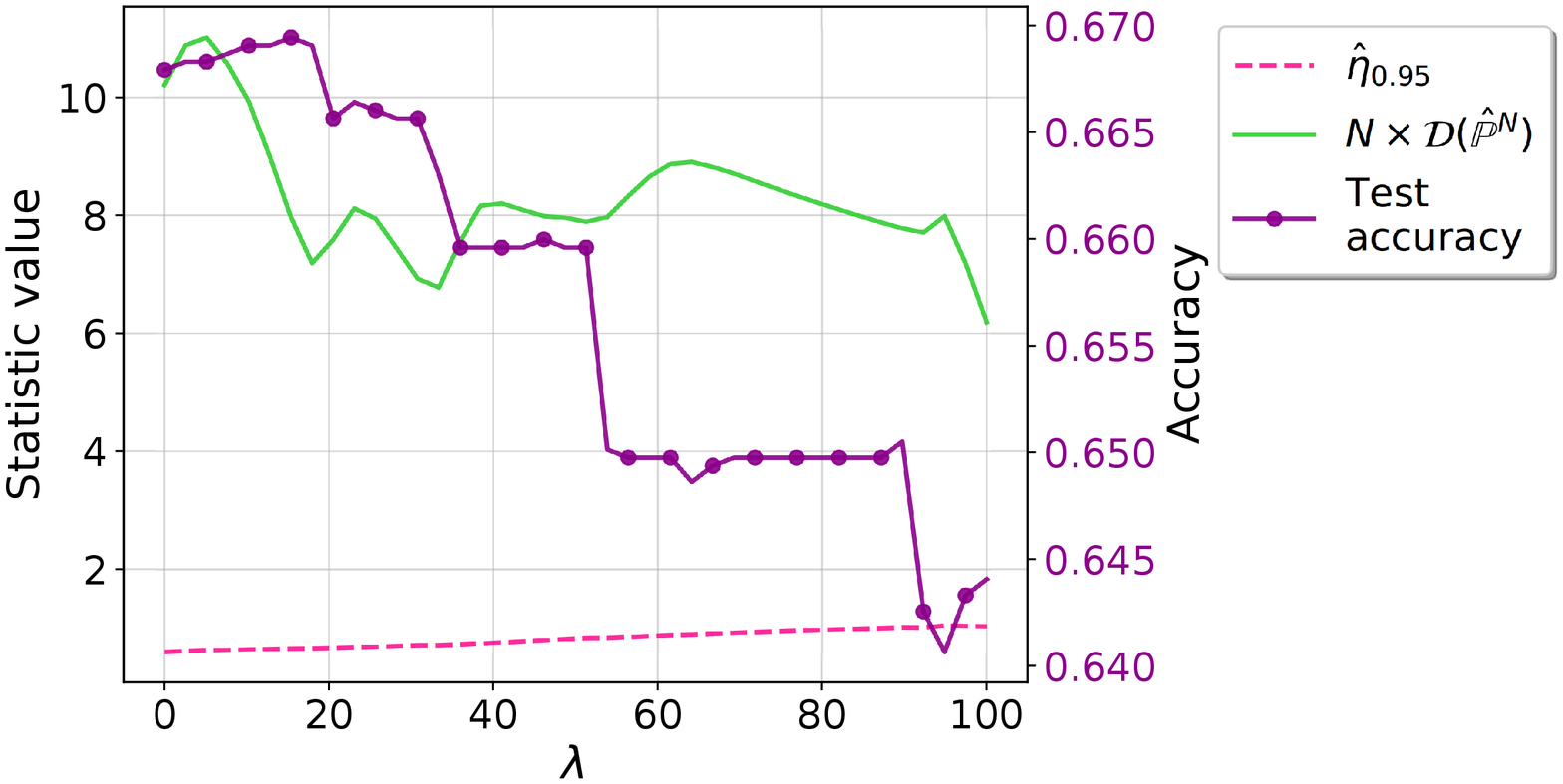}}
\subfigure[SVM classifier in Drug dataset]{
\label{fig:drug} \includegraphics[width=3.5in]{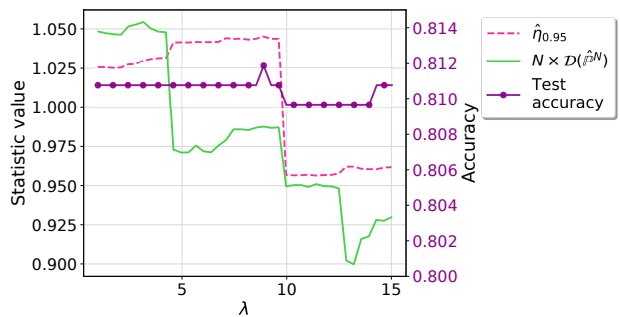}}

\caption{Test statistics and accuracy of regularized classifiers on test data with a rejection threshold. The green line is the test statistics; the pink dashed line is the rejection threshold at the significance level $\alpha = 0.05$; the purple line is the test accuracy;  $\lambda$ denotes the regularization parameter, where larger $\lambda$ means stronger regularization power.}
\label{fig:regularization}
\end{figure}

In the second experiment, we compare a fair algorithm proposed in \citet{ref:donini2018empirical}
with a naive SVM classifier (parametrized by the ridge regularization $\lambda$) in three datasets: Arrhythmia, COMPAS and Drug.
We randomly split 70\%-30\% of the data as a train-test set and we replicate this procedure 1,000 times. We will test the fairness in terms of  the equal opportunity and equalized odds criteria, and for the equal opportunity criteria, we will further show results using Welch's test (Welch's test is not applicable for multi-dimensional equalized odds criteria). Tables \ref{tab:rej_prob_real} and \ref{tab:rej_prob_real2}
show a rejection percentage of the naive SVM and the method in \citet{ref:donini2018empirical}
at the significance level $\alpha = 0.05$ in those 1,000 replications using our test according to the equal opportunity and equalized odds criteria, respectively. Table \ref{tab:rej_prob_real3} shows the test results using  Welch’s test according to the equal opportunity criterion.
Our test results demonstrate that the method in \citet{ref:donini2018empirical} has a significantly lower rejection rate,
which means it is substantially more fair than the naive method.

\begin{table}[htbp]
\caption{Rejection percentage of the Naive SVM and the method in \citet{ref:donini2018empirical} at the significance level $\alpha = 0.05$ according to the equal opportunity criterion using  our test.}
\label{tab:rej_prob_real}\centering
\begin{tabular}{lccc}
\toprule & Arrhythmia & COMPAS & Drug  \\
\midrule Naive SVM&   68.4\% & 100\%   &  30.1\% \\
\citet{ref:donini2018empirical} &   11.6\%& 16.6\%   &  21.6\% \\

\bottomrule
\end{tabular}%
\end{table}
\begin{table}[htbp]
\caption{Rejection percentage of the Naive SVM and the method in Donini et al.~(2018)  at the significance level $\alpha = 0.05$ according to the equalized odds criterion using our test.}
\label{tab:rej_prob_real2}\centering
\begin{tabular}{lccc}
\toprule & Arrhythmia & COMPAS & Drug  \\
\midrule Naive SVM&   75.1\% &100\%   &  30.5\% \\
Donini et al.~(2018) &  13.7\%& 21.7\%   &  17.2\% \\
\bottomrule
\end{tabular}%
\end{table}

\begin{table}[!htbp]
\caption{Rejection percentage of the Naive SVM and the method in Donini et al. (2018)  at the significance level $\alpha = 0.05$ according to the equal opportunity criterion using Welch's test.}
\label{tab:rej_prob_real3}\centering
\begin{tabular}{lccc}
\toprule & Arrhythmia & COMPAS & Drug  \\
\midrule Naive SVM&   76.1\% &100\%   &  35.5\% \\
Donini et al. (2018) &  14.0\%& 16.1\%   &  23.0\% \\
\bottomrule
\end{tabular}%
\end{table}

More experiments are conducted in \ref{sec:numerical_detail} to empirically validate the convergence result in Theorem \ref%
{thm:clt} and our proposed hypothesis test method.
\section*{Acknowledgement}
Material in this paper is based upon work supported by the Air Force Office of Scientific Research under award number FA9550-20-1-0397. Additional support is gratefully acknowledged
from NSF grants 1915967, 1820942, and 1838576 and Singapore Ministry of Education’s AcRF grant MOE2019-T2-2-163.

%%%%%%%%%%%%%%%%%%%%%%%%%%%%%%%%%

\setcounter{section}{0} \setcounter{subsection}{0} \setcounter{equation}{0}
\ \renewcommand
{\thesection}{Appendix \Alph{section}}
\renewcommand\thesubsection{\thesection.\arabic{subsection}}
\renewcommand{\theequation}{\Alph{section}.\arabic{equation}}
\renewcommand{\thelemma}{\Alph{section}\arabic{lemma}}
\renewcommand{\thetheorem}{\Alph{section}\arabic{theorem}} \newpage

\onecolumn

\section{Proofs}
\label{Appenix:proof}
We prove Theorems \ref{thm:clt} and \ref{thm:clt2} under a more general assumption below.
\setcounter{assumption}{2}
\begin{assumprime}{\theassumption}[Continuous conditional measure]
    For the case where  $\text{supp}(U)$ is potentially an
       infinite set, the cost function $c$ is decomposable as
\begin{align*}
  c\left((x,u),(x^\prime,u^\prime)\right) =
  \bar{c}(x,x^{\prime})+\infty \cdot \|u-u^{\prime }\|,
\end{align*}
       and the following conditions are
       satisfied:
    \begin{enumerate}[label=\emph{\alph*})]
    \item the moments $\mathbb{E}_{\mathbb{P}} \| U\|_{2}^{2}$,
      $\mathbb{E}_{\mathbb{P}} \| \phi (U,\mu ) \| _{2}^{2}$ and
      $\mathbb{E}_{\mathbb{P}}\| \phi_{z}(U,\mu )\|_{2}$ are finite.
      \label{assump:1.2moments}
    \item for $z$ such that $\Vert z - \mu \Vert_2 < v,$ the
      derivative $\phi_z(\cdot)$ satisfies,
      \begin{equation}
        \left\Vert \phi _{z}(u,z)-\phi _{z}(u,\mu )\right\Vert _{2}\leq
        M(u)\left\Vert z-\mu \right\Vert _{2},
        \label{lma:assump}
      \end{equation}
      where $\mathbb{E}_{\mathbb{P}}[M(U)]<+\infty.$ \label{assump:2.2}
    \item \label{assump:1.4} The (regular) conditional probability
      measure $\nu_t $ of $\phi ( U, \mu) |\Phi (X) = t$ converges in
      terms of the type $1$-Wasserstein distance as $t\rightarrow 0$:
      i.e., there exist a set $B\subset \mathbb{R}$ with
      $\mathbb{P}( \Phi (X)\in B) =1$ and $\varepsilon_0 > 0$ such
      that
      \begin{equation*}
        \lim_{t\rightarrow 0}W_{1}\left( \nu_t, \nu_0\right)
        \mathbf{1}\{t\in B\}=0
      \end{equation*}
      and
      $\sup_{t \in B}\mathbb{E}_{\mathbb{P}}[ \| \phi ( U,\mu)
      \|_{2}^{2+\epsilon _{0}}|\Phi (X)\mathcal{=}t]$ is finite, where type 1-Wasserstein distance $W_1(\cdot,\cdot)$ is $W_c(\cdot,\cdot)$ with the cost function being a metric.
    \end{enumerate}
    \label{assump:infinite}
\end{assumprime}
\begin{remark}
  \textnormal{If $\text{supp}(U)$ is a finite set, and $U$ is completely dependent on $(X,Y)$,
    the simpler Assumption~\ref{assump:finite} is equivalent to
    Assumption~\ref{assump:infinite}.}
  \label{rem:assump-equiv}
\end{remark}
% In light of Remark~\ref{rem:assump-equiv}, we emphasize that our proof
% Theorem~\ref{thm:clt} under Assumption implies that the simpler
% Assumption~\ref{assump:finite} when $U$ takes only finitely many
% values.  This is indeed the case with all the common fairness notions
% in Examples \ref{eg:EOpp}~-~\ref{eg:SPar}.
\subsection{Proofs of Section~\protect\ref{sec:simple}}
\label{prf:sec_simple}
% To shorten the notation, we write $\xi = (X, U)$ and denote $\Xi = \X \times
% \mathbb{U}$, $\wh \Xi_N = \{(x_i, u_i)\}_{i=1}^N$. We assume that $N \ge 2$
% and $\xi_i = (x_i, u_i)$ are distinct. We use $\mc M_+(\Xi \times \wh \Xi_N)$
% to denote the set of all nonnegative measures on $\Xi \times \wh \Xi_N$.
\begin{proof}[Proof of Proposition~\ref{prop:primal}]Since the cost to move $U$ is $+\infty$, we have  $\EE_{\QQ}[U]=\EE_{\Pnom}[U]$.
  Then, consider any probability measure $\QQ$ such that
    \[
        \EE_{\QQ}[\mc C(X) \phi(U, \EE_{\Pnom}[U])] = 0,
    \]
    and let $\pi$ be the optimal coupling between $\Pnom$ and $\QQ$. Because $\Pnom$ is the empirical measure, the coupling $\pi$ can be written as $\pi = \frac{1}{N} \sum_{i\in[N]} \pi_i \otimes \delta_{(x_i, u_i)}$. For any value $\eps > 0$, construct now the measure
    \begin{equation}
        \QQ^\eps = \frac{1}{N} \sum_{i \in [N]} (1- p_i) \delta_{(x_i, u_i)} + p_i \delta_{(x_i^\varepsilon, u_i)},
        \label{eqn:Q_eps}
    \end{equation}
    where the mass $p_i$ is set to
    \[
        p_i = \int_{\X_{1 - \mc C(x_i)}} \pi_i(\mathrm{d} x) = \pi_i(\X_{1 - \mc C(x_i)}) \in [0, 1] \qquad \forall i \in [N]
    \]
    and ${x}^\varepsilon_i$ is an $\varepsilon$-optimizer of the problem $\inf_{x' \in \X_{1-\mc C(x)}%
} c(x_i,x^{\prime })$. Then, it is easy to see that
\begin{align*}
    &\EE_{\QQ^\eps}[\mc C(X) \phi(U, \EE_{\Pnom}[U])]=\EE_{\QQ}[\mc C(X) \phi(U, \EE_{\Pnom}[U])] = 0
\end{align*}
and that
\begin{align*}
    &\EE_{\QQ^\eps}[\mc C(X) \phi(U, \EE_{\Pnom}[U])] \\
    &=\frac{1}{N} \left(\ds \sum_{i \in [N]} (1-p_i)\mc C(x_i)  \phi(u_i, \EE_{\Pnom}[U])  +  p_i(1 -  \mc C(x_i)) \phi(u_i, \EE_{\Pnom}[U])\right) \\
    &=\frac{1}{N}\left( \ds \sum_{i \in [N]} (1 - 2 \mc C(x_i)) \phi(u_i, \EE_{\Pnom}[U]) p_i + \sum_{i\in [N]}\mc C(x_i) \phi(u_i, \EE_{\Pnom}[U])\right).
    \end{align*} Since  $d(x_i) \le \| x_i^\eps - x_i\| \le \|x -x_i\| + \eps $ for any $x \in \X_{1-\mc C(x_i)}$, this implies that $\frac{1}{N}\sum_{i \in [N]} p_i d(x_i) \leq W(\QQ, \Pnom) + \eps$.
     Since $\eps$ can be chosen arbitrarily, this implies that
    \begin{equation}
        \mc P(\Pnom) \geq \left\{
            \begin{array}{cl}
                \min & \ds \frac{1}{N}\sum_{i \in [N]} p_i d(x_i) \\
                \st & p \in [0, 1]^N \\
                & \ds \sum_{i \in [N]} (1 - 2 \mc C(x_i)) \phi(u_i, \EE_{\Pnom}[U]) p_i = - \sum_{i\in [N]}\mc C(x_i) \phi(u_i, \EE_{\Pnom}[U]).
            \end{array}
        \right.
        \label{problem:primal_temp}
    \end{equation}
    On the other hand, for any $\{p_i\}_{i=1}^N$ satisfying the constraints in the linear programming \eqref{problem:primal_temp}, we can construct the measure $\QQ^\eps$ according to \eqref{eqn:Q_eps}. Since $\QQ^\eps$ is a feasible solution of the primal problem \eqref{problem:primal}, we have the other direction of the inequality.
\end{proof}%%%%%%%%%%%%%%%%%%%%
\begin{proof}[Proof of Proposition \ref{prop:duality}] \textbf{Case 1:} $d(x_i)<+\infty$ for $i \in [N]$. The primal problem has a feasible solution ($p_i = 0$ if $\mc C(X) = 0$; $p_i = 1$, otherwise) and is bounded, thus it has an optimal solution and the strong duality holds.
By the strong duality, we have
\begin{equation}
        \mc P(\Pnom) = \left\{
            \begin{array}{cl}
                \max & \ds \frac{1}{N}\left\{\sum_{i \in [N]} \alpha^i +\gamma^\top \mc C(x_i)  \phi(u_i, \EE_{\Pnom}[U])\right \} \\
                \st & \alpha_i \leq 0 \quad \forall i \in [N] \\
                & \ds \alpha_i -\left(1-2\mc C(x_i) \right)  \gamma^\top \phi(u_i, \EE_{\Pnom}[U])\leq d(x_i) \quad \forall i \in [N].
            \end{array}
        \right.
        \label{problem:dual2}
    \end{equation}
    Then, we have
    \begin{align*}
    \alpha_i = \left(d(x_i) +  \left(1-2\mc C(x_i) \right)\gamma^\top \phi(u_i, \EE_{\Pnom}[U]) \right)^{-},
    \end{align*}
    which gives the desired results.

    \textbf{Case 2:} $\exists i\in [N]$ such that $d(x_i)=+\infty$. The primal problem is equivalent to
    \begin{equation}
\mc P(\Pnom) = \left\{
\begin{array}{cl}
  \min & \ds \frac{1}{N}\sum_{i \in [N]} p_i d(x_i), \\
  \st & p \in [0, 1]^N, \\
  & p_i = 0 \text{ for } d(x_i) = +\infty \\
       & \ds \sum_{i \in [N]} (1 - 2 \mc C(x_i)) \phi(u_i, \EE_{\Pnom}[U]) p_i  = -
         \ds\sum_{i\in [N]}\mc C(x_i) \phi(u_i, \EE_{\Pnom}[U]),%
\end{array}
\right.
\label{problem:reformulate2}
\end{equation}
with the convention that if the problem is infeasible, the optimal value of the minimization problem is $+\infty$. We have the dual problem
\begin{equation}
            \begin{array}{cl}
                \max & \ds \frac{1}{N}\left\{\sum_{d(x_i) < +\infty} \alpha^i +\gamma^\top \mc C(x_i)  \phi(u_i, \EE_{\Pnom}[U])\right \} \\
                \st & \alpha_i \leq 0 \quad \text{ for } d(x_i) < +\infty \\
                & \ds \alpha_i -\left(1-2\mc C(x_i) \right)  \gamma^\top \phi(u_i, \EE_{\Pnom}[U])\leq d(x_i) \quad \text{ for } d(x_i) < +\infty.
            \end{array}
        \label{problem:dual22}
    \end{equation}
    Since the problem \eqref{problem:dual22} is also feasible, if it is bounded, then the strong duality holds. If the problem \eqref{problem:dual22} is unbounded,  the primal problem \eqref{problem:reformulate2} is infeasible, which means the primal and the dual both have optimal value $+\infty$. Finally, because
    \[
    \left(d(x_i) +  \left(1-2\mc C(x_i) \right)\gamma^\top \phi(u_i, \EE_{\Pnom}[U]) \right)^{-}=0 \text{ for } d(x_i) = +\infty,
    \]
    we have the optimal value of problem \eqref{problem:dual22} equals to $\mc D(\Pnom)$.
\end{proof}
\begin{proof}[Proof of Lemma \ref{lma:sufficient_assump_1.1}] Since $X_C\big|X_D=x$ has positive density in  $\mathbb{R}^{d_2}$ for every $x \in \mathbb{D}$, we have $\theta_C^\top X_C\big|X_D=x$ has positive density in  $\mathbb{R}$ for every $x \in \mathbb{D}$. Therefore, $\theta^\top X$ has a density
 \[
 f_{\theta^\top X}(x)=\sum_{v \in \mathbb{D}} p_v f_{\theta_C^\top X_C\big|v}(x-\theta_D^\top v)>0,
 \]
 where $p_v=\PP \left(X_D=v\right)$ and $f_{\theta_C^\top X_C\big|v}(\cdot)$ denotes the conditional density of $\theta_C^\top X_C\big|X_D=x$.

 Further, let $w = \ell^{-1} (\tau)$. For the cost function $\bar{c}(\cdot)$  given by (\ref{eq:norm-c}), we have by H\"older inequality
  \[
 d(x)=\inf_{\theta^\top x'=w}\|x-x'\| =\|\theta\|^{-1}_* |\theta^\top x-w|.
 \]
 Therefore, $\PP_\Phi$ has a continuous density $f(\cdot)=\|\theta\|_* f_{\theta^\top X}(\|\theta\|_*\times\cdot+w)$ with $f(0)>0$.

    For the cost function $\bar{c}(\cdot)$ given by (\ref{discrete-c}), when $d(x) < \delta$, we have
 \begin{align*}
 d(x)=\inf_{\theta^\top x'=w}\bar{c}(x,x') = \inf_{ x_D=x_D',\theta_C^\top x'=w-\theta_D^\top x_D} \bar{c}(x,x')   = \inf_{ \theta_C^\top x'=w-\theta_D^\top x_D} \|x_C-x_C'\|=\|\theta_C\|^{-1}_* |\theta^\top x-w| .
 \end{align*}
 The last equality is again due to H\"older inequality. Therefore, $\PP_\Phi$ has a continuous density $f(\cdot)=\|\theta_C\|_* f_{\theta^\top X}(\|\theta_C\|_*\times \cdot+w)$ with $f(0)>0$, which completes the proof.
 \end{proof}

%\begin{corollary}
%\label{cor:fairness}If
%\begin{equation*}
%\phi \left( (a,y),z\right) =z_{1}^{-1}\mathbb{I}_{(1,1)}(a,y)-z_{2}^{-1}%
%\mathbb{I}_{(0,1)}(a,y),
%\end{equation*}%
%and
%\begin{equation*}
%\psi \left( (a,y)\right) =\left[ \mathbb{I}_{(1,1)}(a,y),\mathbb{I}%
%_{(0,1)}(a,y)\right] ^{\mathrm{\top }}.
%\end{equation*}%
%We recover the definition (\ref{eqn:rwpi_fair}). Thus, if  $\mathbb{E}_{\mathbb{%
%P}}[U]=\mu \in (0,1)^{s}$, we have%
%\begin{equation*}
%N\mathcal{R}^{\mathrm{fair}}\left( \mathbb{\hat{P}}^{N}\right) \Rightarrow
%\frac{1}{2}C^{-1}V^{2},
%\end{equation*}%
%where
%\begin{eqnarray*}
%V &\sim &\left( p_{01}p_{11}\right) ^{-1}\mathcal{N}(0,var(Z)), \\
%C &=&f\left( 0\right) \left( p_{0,1|\tau }p_{0,1}^{-2}+p_{1,1|\tau
%}p_{1,1}^{-2}\right) ,
%\end{eqnarray*}%
%and
%\begin{align*}
%Z&=\mathbb{I\{}h_{}{(X)\geq }\tau \}\left( p_{01}\mathbb{I}%
%_{(1,1)}(A,Y)-p_{11}\mathbb{I}_{(0,1)}(A,Y)\right) \\
%&\qquad +\mathbb{I}_{(0,1)}(A,Y)\mathbb{E}[\mathbb{I}_{(1,1)}(A,Y)\mathbb{I\{%
%}h_{}{(X)\geq }\tau \}] \\
%&\qquad -\mathbb{I}_{(1,1)}(A,Y)\mathbb{E}[\mathbb{I}_{(0,1)}(A,Y)\mathbb{I\{%
%}h_{}{(X)\geq }\tau \}].
%\end{align*}%
%and%
%\begin{equation*}
%p_{0,1|\B }=\mathbb{E}\left[ \mathbb{I}_{(0,1)}(A,Y)|X \in \B
%\right] ,p_{1,1|\B }=\mathbb{E}\left[ \mathbb{I}_{(1,1)}(A,Y)|X \in \B  \right] .
%\end{equation*}
%\end{corollary}
Lemmas \ref{lma:aux_clt} and \ref{lma:lln} are useful for the proof of
Theorem \ref{thm:clt}, whose proofs are presented in \ref{sec:proofs}.

\begin{lemma}
\label{lma:aux_clt}Suppose Assumption \ref{assump:infinite} is
enforced. Then, we have
\begin{eqnarray*}
\sqrt{N}\left( \mathbb{E}_{\mathbb{\hat{P}}^{N}}\left[ \phi \left( U,%
\mathbb{E}_{\mathbb{\hat{P}}^{N}}[U]\right) \mc C(X)\right] -\mathbb{E}_{%
\mathbb{P}}\left[ \phi (U,\mu )\mc C(X)\right] \right) 
\Rightarrow \mathcal{N}(0,\cov\left( \mathbb{E}_{\mathbb{P}}\left[ \phi
_{z}(U,\mu )\mc C(X)\right] U+\phi (U,\mu )\mc C(X)\right) ).
\end{eqnarray*}
\end{lemma}

\begin{lemma}
\label{lma:lln}Suppose Assumption \ref{assumption} and \ref{assump:infinite} are enforced. Then, we have%
\begin{eqnarray*}
\sqrt{N}\mathbb{E}_{\mathbb{\hat{P}}^{N}}\left[ \left( -\gamma ^{\top
}\phi \left( U,\mathbb{E}_{\mathbb{\hat{P}}^{N}}[U]\right) +\sqrt{N}%
d(X)\right) ^{-}\mc C(X)\right] 
\overset{p}{\longrightarrow }-\frac{1}{2}f(0)\mathbb{E}_{\PP}\left[ \left.
\left( \gamma ^{\top }\phi \left( U,\mu\right) \right) ^{2}\mathbb{I}%
\{\gamma ^{\top }\phi \left( U,\mu\right) \geq 0\}\right\vert d(X\mathcal{)}%
=0\right] ,
\end{eqnarray*}%
uniformly over $\left\Vert \gamma \right\Vert _{2}\leq B.$
\end{lemma}

We are now ready to prove Theorem~\ref{thm:clt}.
\begin{proof}[Proof of Theorem \ref{thm:clt}]
Recall that
\begin{eqnarray*}
\mc D( \mathbb{\hat{P}}^{N}) &=&\max_{\gamma \in
\mathbb{R}^{m}}\frac{1}{N}\left\{\sum_{i \in [N]} \left(d(x_i) +  \left(1-2\mc C(x_i) \right)\gamma^\top \phi(u_i, \EE_{\Pnom}[U]) \right)^{-}+\gamma^\top \mc C(x_i)  \phi(u_i, \EE_{\Pnom}[U])\right \} \\
&=&\max_{\gamma \in \mathbb{R}^{m}} \left\{\frac{1}{N}\sum_{i \in [N]} \gamma^\top \mc C(x_i)  \phi(u_i, \EE_{\Pnom}[U])\right.+ \\
&&\left.\frac{1}{N} \sum_{i \in [N]} \left(d(x_i) - \gamma^\top  \phi(u_i, \EE_{\Pnom}[U]) \right)^{-}\mathcal{C}(x_i) + \left(d(x_i) + \gamma^\top  \phi(u_i, \EE_{\Pnom}[U]) \right)^{-}(1-\mathcal{C}(x_i))\right \}.
\end{eqnarray*}%
 We first rescale $\gamma \leftarrow \gamma \sqrt{N}$ and thus
\begin{eqnarray*}
&&N\mc D( \mathbb{\hat{P}}^{N}) \\
&=&\sqrt{N}\max_{\gamma \in \mathbb{R}^{m}}\left\{ \gamma ^{\top }\mathbb{E}%
_{\mathbb{\hat{P}}^{N}}\left[ \phi \left( U,\mathbb{E}_{\mathbb{\hat{P}}%
^{N}}[U ]\right)\mc C(X) \right] \right. + \\
&&\left. \mathbb{E}_{\mathbb{\hat{P}}^{N}}\left[ \left(\sqrt{N}d(X) - \gamma^\top  \phi(U, \EE_{\Pnom}[U]) \right)^{-}\mathcal{C}(X) + \left(\sqrt{N}d(X) + \gamma^\top  \phi(U, \EE_{\Pnom}[U]) \right)^{-}(1-\mathcal{C}(X)) \right] \right\} .
\end{eqnarray*}%

To ease the notation, we denote $\lambda _{i}=\phi \left( u_{i},\mathbb{E}_{%
\mathbb{\hat{P}}^{N}}[U ]\right) $. By Lemma \ref{lma:lln}, we have
\begin{eqnarray*}
&&\frac{1}{\sqrt{N}}\sum_{i=1}^{N}\left( -\gamma ^{\top }\lambda _{i}+N^{1/2}d %
(x_{i})\right) ^{-}\mc (1-\mc C(x_i)) \\
&\overset{p}{\longrightarrow }&-\frac{1}{2}f(0)\mathbb{E}_{\PP}\left[ \left.
\left( \gamma ^{\top }\phi \left( U,\mu\right) \right) ^{2}\mathbb{%
I}\{\gamma ^{\top }\phi \left( U,\mu\right) \geq 0\}\right\vert d(X\mathcal{)}=0\right] \\
&=&-\frac{1}{2}f(0)\mathbb{E}_{\PP}\left[ \left. \left( \gamma ^{\top }\phi \left(
U,\mu\right) \right) ^{2}\mathbb{I}\{\gamma ^{\top }\phi \left( U,%
\mu\right) \geq 0\}\right\vert d(X)=0 %
\right] ,
\end{eqnarray*}%
and similarly, we have
\begin{equation*}
\frac{1}{\sqrt{N}}\sum_{i=1}^{N}\left( \gamma ^{\top }\lambda _{i}+N^{1/2}d %
(x_{i})\right) ^{-} (1-\mc C(x_i)) \overset{p}{\longrightarrow }-\frac{1}{2}f(0)\mathbb{E}_{\PP}\left[
\left. \left( \gamma ^{\top }\phi \left( U,\mu\right) \right) ^{2}%
\mathbb{I}\{\gamma ^{\top }\phi \left( U,\mu\right)
<0\}\right\vert d(X)=0 \right] .
\end{equation*}%
Therefore, we have
\begin{eqnarray*}
&&\sqrt{N}\mathbb{E}_{\mathbb{\hat{P}}^{N}}\left[ \left(\sqrt{N}d(X) - \gamma^\top  \phi(U, \EE_{\Pnom}[U]) \right)^{-}\mathcal{C}(X) + \left(\sqrt{N}d(X) + \gamma^\top  \phi(U, \EE_{\Pnom}[U]) \right)^{-}(1-\mathcal{C}(X)) \right] \\
&&\overset{p}{\longrightarrow }-\frac{1}{2}f(0)\mathbb{E}_{\PP}\left[ \left.
\left( \gamma ^{\top }\phi \left( U,\mu\right) \right)
^{2}\right\vert d(X)=0 \right] .
\end{eqnarray*}%
We denote
\begin{eqnarray*}
V_{N} &=&\sqrt{N}\mathbb{E}_{\mathbb{\hat{P}}^{N}}\left[ \phi \left( U,%
\mathbb{E}_{\mathbb{\hat{P}}^{N}}[U ]\right) \mc C(X)\right] ,\text{ and} \\
M_{N}(\gamma ) &=&\frac{1}{\sqrt{N}}\sum_{i=1}^{N}\left[ \left( -\gamma
^{\top }\lambda _{i}+N^{1/2}d(x_{i})\right) ^{-}\mc C(X) \}+\left( \gamma ^{\top
}\lambda _{i}+N^{1/2}d(x_{i})\right) ^{-}(1-\mc C(X))\right] \text{ .}
\end{eqnarray*}%
To proceed, we rely on the following lemma.

\begin{lemma}
\label{lma:compactifaction}Suppose Assumption \ref{assumption} is enforced.
Then, for every $\varepsilon >0,$ there exists $N_{0}>0$ and $b\in (0,\infty
)$ such that for all $N\geq $ $N_{0}$,
\begin{equation*}
\mathbb{P}\left( \sup_{\left\Vert \gamma \right\Vert _{2}>b}\left\{ \gamma
^{\top }V_{N}+M_{N}(\gamma )\right\} >0\right) \leq \varepsilon .
\end{equation*}%
\end{lemma}

The proof of Lemma \ref{lma:compactifaction} is furnished in \ref%
{sec:proofs}. Notice that $\mc D( \Pnom) \geq 0$ (choosing $\gamma =0),$ Lemma \ref%
{lma:compactifaction} implies that when $N\geq N_{0}$,
\begin{equation*}
\mathbb{P}\left\{ N\mc D(\Pnom)
=\sup_{\left\Vert \gamma \right\Vert _{2}\leq b}\left\{ \gamma ^{\top
}V_{N}+M_{N}(\gamma )\right\} \right\} \geq 1-\varepsilon .
\end{equation*}%
By Lemmas \ref{lma:aux_clt} and \ref{lma:lln}, we have%
\begin{eqnarray*}
\sup_{\left\Vert \gamma \right\Vert _{2}\leq b}\left\{ \gamma ^{\top
}V_{N}+M_{N}(\gamma )\right\} &\Rightarrow &\sup_{\left\Vert \gamma
\right\Vert _{2}\leq b}\left\{ \gamma ^{\top }V-\frac{1}{2}f(0)\mathbb{E}%
\left[ \left. \left( \gamma ^{\top }\phi \left( U,\mu\right)
\right) ^{2}\right\vert d(X)=0 \right] \right\}
\\
&=&\sup_{\left\Vert \gamma \right\Vert _{2}\leq b}\left\{ \gamma ^{\top }V-%
\frac{1}{2}\gamma ^{\top }S\gamma \right\} ,
\end{eqnarray*}%
where
\begin{equation*}
S=f(0)\mathbb{E}_{\PP}\left[ \left. \phi \left( U,\mu\right) \phi \left(
U,\mu\right) ^{\top }\right\vert d(X)=0\right] ,
\end{equation*}%
and $V$ is normally distributed with mean zero and covariance matrix
\begin{equation*}
\cov\left( \mathbb{E}_{\mathbb{P}}\left[ \phi _{z}(U,\mu )\mathcal{C}(X)\right] U+\phi (U,\mu )%
\mc C(X)\right) .
\end{equation*}%
By the arbitrariness of $\varepsilon ,$ we have the desired result:%
\begin{equation*}
N\times \mc D(\Pnom) \Rightarrow
\sup_{\gamma }\left\{ \gamma ^{\top }V-\frac{1}{2}\gamma ^{\top }S\gamma
\right\} .
\end{equation*}
This completes the proof.
\end{proof}

\subsection{Proofs of Section~\protect\ref{sec:composite}}
The proofs of Propositions \ref{prop:primal_eps} and \ref{prop:duality2} are not presented because they follow the same lines as the proofs of
Propositions~\ref{prop:primal} and \ref{prop:duality}.
%\begin{proof}[Proof of Proposition \ref{prop:duality2}]
%By the duality of the linear programming, we have the dual problem admits the form
%\begin{equation*}
%            \begin{array}{cl}
%                \max & \ds \frac{1}{N}\left\{\sum_{i \in [N]} \alpha^i +\gamma^\top \mc C(x_i) \left( \phi(u_i, \EE_{\Pnom}[U])-\epsilon\right) \right \} \\
%                \st & \alpha_i \leq 0 \quad \forall i \in [N], \quad \gamma \geq 0 \\
%                & \ds \alpha_i -\left(1-2\mc C(x_i) \right)  \gamma^\top \phi(u_i, \EE_{\Pnom}[U])\leq d(x_i) \quad \forall i \in [N]
%            \end{array}
%        %\label{problem:dual_eps2}
%    \end{equation*}
%    Then, by following a similar line of the proof of Proposition \ref{prop:duality}, we have the desired result.%we have
%\end{proof}
\begin{proof}[Proof of Theorem~\ref{thm:clt2}] Let
${\epsilon }^{\ast }=\mathbb{E}_{\mathbb{P}}\left[ \mc C(X)\phi \left( U,\mathbb{E}_{\mathbb{P}}[U]\right) \right]$. By following the similar arguments with the proof of Theorem \ref{thm:clt}, we have
\begin{eqnarray*}
&&N\mc D_{\epsilon }(\mathbb{\hat{P}}^{N}) \\
&=&N\sup_{\gamma \in \mathbb{R}_{+}^{m}}\left\{ -\gamma ^{\top }\epsilon +%
\frac{1}{N}\left\{ \sum_{i\in \lbrack N]}(1-2\mc C(x_{i}))\phi (u_{i},\EE_{%
\Pnom}[U])+\sum_{i\in \lbrack N]}\mc C(x_{i})\phi (u_{i},\EE_{\Pnom%
}[U])\right\} \right\}  \\
&=&\sqrt{N}\sup_{\gamma  \in \mathbb{R}_{+}^{m}}\left\{ \gamma ^{\top }\left(
\mathbb{E}_{\mathbb{\hat{P}}^{N}}\left[ \phi \left( U,\mathbb{E}_{\mathbb{%
\hat{P}}^{N}}[U]\right) \mc C(X)\}-\epsilon \right] \right) \right. + \\
&&\left. \mathbb{E}_{\mathbb{\hat{P}}^{N}}\left[ \left( -\gamma ^{\top }\phi
(U,\EE_{\Pnom}[U])+\sqrt{N}d(X)\right) ^{-}\mc C(X)+\left( \gamma ^{\top
}\phi (U,\EE_{\Pnom}[U])+\sqrt{N}d(X)\right) ^{-}\left( 1-\mc C(X)\right) %
\right] \right\}  \\
&=&\sup_{\gamma \in \mathbb{R}_{+}^{m}}\left\{ \gamma ^{\top }V_{N}+\sqrt{N}\gamma
^{\top }\left( \epsilon ^{\ast }-\epsilon \right) +M_{N}(\gamma )\right\} .
\end{eqnarray*}%
%\viet{sorry but I don't really see how to get to the last equality. What manipulation did you do to have $\epsilon\opt$ show up?}\nian{add and substract $\epsilon^*$}
Similarly, we still have
\begin{equation*}
\gamma ^{\top }V_{N}+M_{N}(\gamma )\Rightarrow \gamma ^{\top }V-\frac{1}{2}%
\gamma ^{\top }S\gamma ,
\end{equation*}%
uniformly over $\left\{ \gamma :\gamma  \in \mathbb{R}^m_+,\left\Vert \gamma
\right\Vert _{2}\leq B\right\} $. Therefore, we must enforce $\gamma ^{\top
}\left( \epsilon ^{\ast }-\epsilon \right) =0$ here. Then, we have
\begin{equation*}
N\mc D_{\epsilon }(\mathbb{\hat{P}}^{N})\Rightarrow \max_{\gamma \in \mathbb{R}^m_+,\gamma ^{\top }\left( \epsilon ^{\ast }-\epsilon \right)
=0}\left\{ \gamma ^{\top }V-\frac{1}{2}\gamma ^{\top }S\gamma \right\}
\preceq \max_{ \gamma \in \mathbb{R}_{+}^{m}}\left\{ \gamma ^{\top }V-\frac{1}{2}%
\gamma ^{\top }S\gamma \right\} .
\end{equation*}%
This completes the proof. \end{proof}

\subsection{Proofs of Technical Results}

\label{sec:proofs}

\begin{proof}[Proof of Lemma \ref{lma:aux_clt}]
By adding and subtracting the term~$\mathbb{E}_{\mathbb{\hat{P}}^{N}}\left[ \phi (U,\mu )\mathcal{C}(X)\right]$, we find
\be
\begin{array}{rl}
&\mathbb{E}_{\mathbb{\hat{P}}^{N}}\left[ \phi (U,\mathbb{E}_{\mathbb{\hat{P}%
}^{N}}[U])\mc C(X)%
\right] -\mathbb{E}_{\mathbb{P}}\left[ \phi (U,\mu )\mathcal{C}(X)\right]   \label{eqn:conv_1} \\
=&\mathbb{E}_{\mathbb{\hat{P}}^{N}}\left[ \phi (U,\mathbb{E}_{\mathbb{\hat{P%
}}^{N}}[U])\mathcal{C}(X)-\phi (U,\mu )\mathcal{C}(X)
\right] +\mathbb{E}_{\mathbb{\hat{P}}^{N}}\left[ \phi (U,\mu )\mathcal{C}(X)\right] -\mathbb{E}_{\mathbb{P}}%
\left[ \phi (U,\mu)\mathcal{C}(X)\right]
\end{array}
\ee
Under Assumption \ref{assump:infinite} and the fundamental theorem of calculus, the first term in the right-hand side of~\eqref{eqn:conv_1} becomes
\begin{eqnarray*}
\mathbb{E}_{\mathbb{\hat{P}}^{N}}\left[ \phi (U,\mathbb{E}_{\mathbb{\hat{P}%
}^{N}}[U])\mathcal{C}(X)-\phi (U,\mu )\mc C(X)%
\right]  =\mathbb{E}_{\mathbb{\hat{P}}^{N}}\left[ \int_{0}^{1}\phi _{z}\left( U,\mu
+t\left( \mathbb{E}_{\mathbb{\hat{P}}^{N}}[U]-\mu \right) \right)
\left( \mathbb{E}_{\mathbb{\hat{P}}^{N}}[U]-\mu \right) \mathcal{C}(X)\mathrm{d}t\right] .
\end{eqnarray*}%
Thanks to Assumption \ref{assump:infinite}, we have that
\begin{eqnarray*}
&&\left\Vert \mathbb{E}_{\mathbb{\hat{P}}^{N}}\left[ \int_{0}^{1}\phi
_{z}\left( U,\mu +t\left( \mathbb{E}_{\mathbb{\hat{P}}^{N}}[U]-\mu
\right) \right) \left( \mathbb{E}_{\mathbb{\hat{P}}^{N}}[U]-\mu
\right) \mc C(X)\mathrm{d}t%
\right] \right.  \\
&&\left. \qquad -\mathbb{E}_{\mathbb{\hat{P}}^{N}}\left[ \int_{0}^{1}\phi
_{z}\left( U,\mu \right) \left( \mathbb{E}_{\mathbb{\hat{P}}^{N}}[\psi
(U)]-\mu \right) \mc C(X)%
\mathrm{d}t\right] \right\Vert _{2} \\
&\leq &\frac{1}{2}\mathbb{E}_{\mathbb{\hat{P}}^{N}}\left[ M(U)\right]
\left\Vert \mathbb{E}_{\mathbb{\hat{P}}^{N}}[U]-\mu \right\Vert
_{2}^{2},
\end{eqnarray*}%
whenever $\left\Vert \mathbb{E}_{\hat{\mathbb{P}}^{N}}[U]-\mu \right\Vert
_{2}<\varepsilon _{\mu }$. Then, notice that we have
\begin{equation}
\lim_{N\rightarrow \infty }\frac{1}{2}\sqrt{N}\mathbb{E}_{\mathbb{\hat{P}}%
^{N}}\left[ M(U)\right] \left\Vert \mathbb{E}_{\mathbb{\hat{P}}^{N}}[\psi
(U)]-\mu \right\Vert _{2}^{2}=0\text{ almost surely,}
\label{eqn:o_p_1}
\end{equation}%
and
\begin{eqnarray*}
\mathbb{E}_{\mathbb{\hat{P}}^{N}}\left[ \int_{0}^{1}\phi _{z}\left( U,\mu
\right) \left( \mathbb{E}_{\mathbb{\hat{P}}^{N}}[U]-\mu \right)
\mc C(X)\mathrm{d}t\right] &=&\mathbb{E}_{\mathbb{\hat{P}}^{N}}\left[ \phi _{z}\left( U,\mu \right)
\mc C(X)\right] \left(
\mathbb{E}_{\mathbb{\hat{P}}^{N}}[U]-\mu \right) \} \\
&=&\left( \mathbb{E}_{\mathbb{P}}\left[ \phi _{z}\left( U,\mu \right)
\mc C(X)\right] +o_{p}\left(
1\right) \right) \left( \mathbb{E}_{\mathbb{\hat{P}}^{N}}[\psi
(U)]-\mu \right).
\end{eqnarray*}%
By multiplying $\sqrt{N}$ to both sides of equation~\eqref{eqn:conv_1}, we have
\begin{eqnarray*}
&&\sqrt{N}\left( \mathbb{E}_{\mathbb{\hat{P}}^{N}}\left[ \phi \left( U,%
\mathbb{E}_{\mathbb{\hat{P}}^{N}}[U]\right) \mathcal{C}(X)\right] -\mathbb{E}_{\mathbb{P}}\left[ \phi
(U,\mu )\mc C(X)\right]
\right)  \\
&=&\sqrt{N}\left( \mathbb{E}_{\mathbb{P}}\left[ \phi _{z}\left( U,\mu
\right) \mc C(X)\right]
+o_{p}\left( 1\right) \right) \left( \mathbb{E}_{\mathbb{\hat{P}}%
^{N}}[U]-\mu \right)  +\mathbb{E}_{\mathbb{\hat{P}}^{N}}\left[ \phi (U,\mu )\mathcal{C}(X)\right] -\mathbb{E}_{\mathbb{P}}%
\left[ \phi (U,\mu)\mathcal{C}(X)\right] +o_{p}(1) \\
&=&\sqrt{N}\mathbb{E}_{\mathbb{\hat{P}}^{N}}\Big[ \mathbb{E}_{\mathbb{P}}%
\big[ \phi _{z}\left( U,\mu \right) \mathcal{C}(X)\big] \left( U-\mu \right) +\phi (U,\mu )%
\mc C(X)-\mathbb{E}_{\mathbb{P%
}}\big[ \phi (U,\mu)\mathcal{C}(X)\big] \Big] +o_{p}(1) \\
&\Rightarrow &\mathcal{N}(0,\Sigma ),
\end{eqnarray*}%
where $\Sigma$ is the covariance matrix of
$\mathbb{E}_{\mathbb{P}}\left[ \phi _{z}(U,\mu )\mathcal{C}(X)\right] U+\phi (U,\mu )%
\mc C(X)$, namely
\begin{equation*}
\Sigma =\cov\left( \mathbb{E}_{\mathbb{P}}\left[ \phi _{z}(U,\mu )\mathcal{C}(X)\right] U+\phi (U,\mu )%
\mc C(X)\right) .
\end{equation*}%
This completes the proof.
\end{proof}

\begin{proof}[Proof of Lemma \ref{lma:lln}] \textbf{Step 1:} we first show
\begin{eqnarray*}
\sqrt{N}\mathbb{E}_{\mathbb{\hat{P}}^{N}}\left[ \left( -\gamma ^{\top
}\phi \left( U,\mathbb{E}_{\mathbb{\hat{P}}^{N}}[U]\right) +\sqrt{N}%
d(X)\right) ^{-}\mathcal{C}(X)\right] -\sqrt{N}\mathbb{E}_{\mathbb{\hat{P}}^{N}}\left[ \left( -\gamma ^{\top
}\phi \left( U,\mu\right) +\sqrt{N} d(X) \right) ^{-}\mathcal{C}(X)\right]
\overset{p}{\longrightarrow }0,
\end{eqnarray*}%
uniformly over $\left\Vert \gamma \right\Vert _{2}\leq B.$ When $\left\Vert
\mathbb{E}_{\mathbb{\hat{P}}^{N}}[U]-\mu\right\Vert _{2}<\varepsilon _{\mu },
$ we have
\begin{eqnarray*}
&&\sqrt{N}\mathbb{E}_{\mathbb{\hat{P}}^{N}}\left[ \left( -\gamma ^{\top
}\phi \left( U,\mathbb{E}_{\mathbb{\hat{P}}^{N}}[U]\right) +\sqrt{N}%
d(X)\right) ^{-}\mathcal{C}(X)\right] -\sqrt{N}\mathbb{E}_{\mathbb{\hat{P}}^{N}}\left[ \left( -\gamma ^{\top
}\phi \left( U,\mu\right) +\sqrt{N}d(X)\right) ^{-}\mc C(X)\right] \\
&\leq &N^{-1/2}\left\Vert \gamma \right\Vert _{2}\sum_{i=1}^{N}\left[
\left\Vert \phi \left( u_{i},\mathbb{E}_{\mathbb{\hat{P}}^{N}}[\psi
(U)]\right) -\phi \left( u_{i},\mu\right) \right\Vert _{2}\mathbb{I\{}
\mathcal{E}_{i}\}\right] ,
\end{eqnarray*}%
where the events $\mathcal{E}_i$ are defined by
\begin{equation*}
\mathcal{E}_{i}=\mathbb{\{}\left\Vert \gamma \right\Vert _{2}\left(
\left\Vert \phi \left( u_{i},\mu\right) \right\Vert _{2}+\left( \left\Vert
\phi _{z}\left( u_{i},\mu\right) \right\Vert _{2}+M(u_i)\varepsilon _{\mu }\right)
\varepsilon _{\mu }\right) {\geq }\sqrt{N}d(x_{i})\}.
\end{equation*}%
By a similar derivation with the proof of Lemma \ref{lma:aux_clt}, we have
\begin{eqnarray*}
&&N^{-1/2}\left\Vert \gamma \right\Vert _{2}\sum_{i=1}^{N}\left[ \left\Vert
\phi \left( u_{i},\mathbb{E}_{\mathbb{\hat{P}}^{N}}[U]\right) -\phi \left(
u_{i},\mu\right) \right\Vert _{2}^{-}\mathbb{I\{}\mathcal{E}_{i}\}\right] \\
&=&\frac{\left\Vert \gamma \right\Vert _{2}}{\sqrt{N}}\sum_{i=1}^{N}\left[
\int_{0}^{1}\left( \phi _{z}\left( u_{i},\mu +t\left( \mathbb{E}_{\mathbb{%
\hat{P}}^{N}}[U]-\mu \right) \right) \left( \mathbb{E}_{\mathbb{\hat{P}}%
^{N}}[U]-\mu \right) \mathrm{d}t\right) \mathbb{I\{}\mathcal{E}_{i}\}\right]
\\
&\leq &\left\Vert \gamma \right\Vert _{2}\sqrt{N}\left( \mathbb{E}_{\mathbb{%
\hat{P}}^{N}}[U]-\mu \right) \mathbb{E}_{\mathbb{\hat{P}}^{N}}\left[ \phi
_{z}\left( U,\mu \right) \mathbb{I}\{ \mathcal{E}_{i}\}\right] +\frac{1}{2}%
\left\Vert \gamma \right\Vert _{2}\sqrt{N}\mathbb{E}_{\mathbb{\hat{P}}%
^{N}}[M(U)]\left\Vert \mathbb{E}_{\mathbb{\hat{P}}^{N}}[U]-\mu \right\Vert
_{2}^{2}.
\end{eqnarray*}%
Since $\mathbb{I}\{ \mathcal{E}_{i}\}\rightarrow 0$ almost surely and $%
\mathbb{E}_{\mathbb{P}}[\phi _{z}(U,\mu )]<+\infty ,$ we have
\begin{equation*}
\left\Vert \gamma \right\Vert _{2}\sqrt{N}\left( \mathbb{E}_{\mathbb{\hat{P}}%
^{N}}[U]-\mu \right) \mathbb{E}_{\mathbb{\hat{P}}^{N}}\left[ \phi _{z}\left(
U,\mu \right) \mathbb{I}\{ \mathcal{E}_{i}\}\right] \overset{p}{%
\longrightarrow }0,
\end{equation*}%
uniformly over $\left\Vert \gamma \right\Vert _{2}\leq B.$ By combining
\begin{equation*}
\frac{1}{2}\left\Vert \gamma \right\Vert _{2}\sqrt{N}\mathbb{E}_{\mathbb{%
\hat{P}}^{N}}[M(U)]\left\Vert \mathbb{E}_{\mathbb{\hat{P}}^{N}}[
U]-\mu \right\Vert _{2}^{2}\rightarrow 0\text{ almost surely,}
\end{equation*}%
uniformly over $\left\Vert \gamma \right\Vert _{2}\leq B,$ we finish step 1.

\textbf{Step 2:} We claim that
\begin{eqnarray*}
\sqrt{N}\mathbb{E}_{\mathbb{P}}\left[ \left( -\gamma ^{\top }\phi \left(
U,\mu \right) +\sqrt{N}d(X)\right) ^{-}\mc C(X)\right]  \longrightarrow -\frac{1}{2}f(0)\mathbb{E}\left[ \left. \left( \gamma
^{\top }\phi \left( U,\mu \right) \right) ^{2}\mathbb{I}\{\gamma ^{\top
}\phi \left( U,\mu \right) \geq 0\}\right\vert d(X)=0\right] .
\end{eqnarray*}%
Notice that for any $c>0$, we have
\begin{eqnarray}
&&\sqrt{N}\mathbb{E}_{\mathbb{P}}\left[ \left( -\gamma ^{\top }\phi \left(
U,\mu \right) +\sqrt{N}d(X)\right) ^{-}\mathcal{C}(X)\right]   \notag \\
&=&\sqrt{N}\int_{0}^{+\infty }\mathbb{E}_{\mathbb{P}}\left[ \left. \left(
-\gamma ^{\top }\phi \left( U,\mu \right) +\sqrt{N}d(X)\right)
^{-}\right\vert d(X)\left( 2\mathcal{C}(X)-1\right) =t\right] \mathrm{d}%
\mathbb{P}_{\Phi }(t).  \notag \\
&=&\sqrt{N}\int_{0}^{c/\sqrt{N}}\mathbb{E}_{\mathbb{P}}\left[ \left. \left(
-\gamma ^{\top }\phi \left( U,\mu \right) +\sqrt{N}d(X)\right)
^{-}\right\vert \Phi (X)=t\right] \mathrm{d}\mathbb{P}_{\Phi }(t)
\label{eqn:2terms_c} \\
&&+\sqrt{N}\int_{c/\sqrt{N}}^{+\infty }\mathbb{E}_{\mathbb{P}}\left[ \left.
\left( -\gamma ^{\top }\phi \left( U,\mu \right) +\sqrt{N}d(X)\right)
^{-}\right\vert \Phi (X)=t\right] \mathrm{d}\mathbb{P}_{\Phi }(t)
\end{eqnarray}%
We first analyze the first term in~\eqref{eqn:2terms_c}. By Assumption \ref%
{assumption}.\ref{assump:1.1}, when $N$ is sufficient large such that $c/\sqrt{N} <v$, we have
\begin{eqnarray*}
&&\sqrt{N}\int_{0}^{c/\sqrt{N}}\mathbb{E}_{\mathbb{P}}\left[ \left. \left(
-\gamma ^{\top }\phi \left( U,\mu \right) +\sqrt{N}d(X)\right)
^{-}\right\vert \Phi (X)=t\right] \mathrm{d}\mathbb{P}_{\Phi }(t) \\
&=&\sqrt{N}\int_{0}^{c/\sqrt{N}}\mathbb{E}_{\mathbb{P}}\left[ \left. \left(
-\gamma ^{\top }\phi \left( U,\mu \right) +\sqrt{N}d(X)\right)
^{-}\right\vert \Phi (X)=t\right] f(t)\mathrm{d}t.
\end{eqnarray*}%
By changing of the variable $s=\sqrt{N}t$ , we have%
\begin{eqnarray*}
&&\sqrt{N}\int_{0}^{c/\sqrt{N}}\mathbb{E}_{\mathbb{P}}\left[ \left. \left(
-\gamma ^{\top }\phi \left( U,\mu \right) +\sqrt{N}d_{\B}(X)\right)
^{-}\right\vert \Phi (X)=t\right] f(t)\mathrm{d}t \\
&=&\int_{0}^{c}\mathbb{E}_{\mathbb{P}}\left[ \left. \left( -\gamma ^{\top
}\phi \left( U,\mu \right) +s\right) ^{-}\right\vert \Phi (X)=N^{-1/2}s%
\right] f(N^{-1/2}s)\mathrm{d}s
\end{eqnarray*}%
By Assumption \ref{assumption}.\ref{assump:1.4}, we have for any $%
\varepsilon >0,$ any $0<c<+\infty ,$ there exists $N_{0},$ such that for $%
N>N_{0}$ and $s\leq c$,
\begin{eqnarray*}
&&\left\vert \mathbb{E}_{\mathbb{P}}\left[ \left. \left( -\gamma ^{\top
}\phi \left( U,\mu \right) +s\right) ^{-}\right\vert \Phi (X)=N^{-1/2}s%
\right] -\mathbb{E}_{\mathbb{P}}\left[ \left. \left( -\gamma ^{\top }\phi
\left( U,\mu \right) +s\right) ^{-}\right\vert \Phi (X)=0\right] \right\vert
\\
&\leq &\left\Vert \gamma \right\Vert _{\ast }\left\Vert \mathbb{E}_{\mathbb{P%
}}\left[ \left. \phi \left( U,\mu \right) \right\vert \Phi (X)=N^{-1/2}s%
\right] -\mathbb{E}_{\mathbb{P}}\left[ \left. \phi \left( U,\mu \right)
\right\vert \Phi (X)=0\right] \right\Vert  \\
&\leq &\left\Vert \gamma \right\Vert _{\ast }W_{1}\left( \left. \phi \left(
U,\mu \right) \right\vert \Phi (X)=N^{-1/2}s,\left. \phi \left( U,\mu
\right) \right\vert \Phi (X)=0\right) \leq \varepsilon .
\end{eqnarray*}%
Therefore, by taking $\varepsilon \downarrow 0,$ we have%
\begin{eqnarray*}
&&\left\vert \int_{0}^{c}\mathbb{E}_{\mathbb{P}}\left[ \left. \left( -\gamma
^{\top }\phi \left( U,\mu \right) +s\right) ^{-}\mathbb{I\{}h_{{}}\left( {X}%
\right) {\geq }\tau \}\right\vert \Phi (X)=N^{-1/2}s\right] f(N^{-1/2}s)%
\mathrm{d}s\right.  \\
&&\left. -\int_{0}^{c}\mathbb{E}_{\mathbb{P}}\left[ \left. \left( -\gamma
^{\top }\phi \left( U,\mu \right) +s\right) ^{-}\mathbb{I\{}h_{{}}\left( {X}%
\right) {\geq }\tau \}\right\vert \Phi (X)=0\right] f(N^{-1/2}s)\mathrm{d}%
s\right\vert  \overset{p}{\longrightarrow }0.
\end{eqnarray*}%
Then, the basic algebra and the mean value theorem for integrals give us
\begin{align}
&\int_{0}^{c}\mathbb{E}_{\mathbb{P}}\left[ \left. \left( -\gamma ^{\top
}\phi \left( U,\mu \right) +s\right) ^{-}\right\vert \Phi (X)=0\right]
f(N^{-1/2}s)\mathrm{d}s  \notag \\
=&\int_{0}^{c}\mathbb{E}_{\mathbb{P}}\left[ \left. \left( -\gamma ^{\top
}\phi \left( U,\mu \right) +s\right) ^{-}\right\vert d(X)=0\right]
f(N^{-1/2}s)\mathrm{d}s  \notag \\
=&f(\xi )\mathbb{E}_{\mathbb{P}}\left[ \int_{0}^{c}\left( \left. \left(
-\gamma ^{\top }\phi \left( U,\mu \right) +s\right) ^{-}\right\vert
d(X)=0\right) \mathrm{d}s\right]   \notag \\
=&f(\xi )\mathbb{E}_{\mathbb{P}}\left[ \int_{0}^{\min (c,\gamma ^{\top
}\phi \left( U,\mu \right) |\Phi (X)=0)}\left( \left. \left( -\gamma ^{\top
}\phi \left( U,\mu \right) +s\right) \mathbb{I}\{\gamma ^{\top }\phi \left(
U,\mu \right) \geq 0\}\right\vert d(X)=0\right) \mathrm{d}s\right]   \notag
\\
\rightarrow &-\frac{1}{2}f(0)\mathbb{E}_{\mathbb{P}}\left[ \min \{c,\gamma
^{\top }\phi \left( U,\mu \right) \}\left( \gamma ^{\top }\phi \left( U,\mu
\right) + \left(\gamma ^{\top }\phi \left( U,\mu \right) -c\right)^+
\right) \mathbb{I}\{\gamma ^{\top }\phi \left( U,\mu \right) \geq 0\}\big|d(X)=0%
\right] ,  \label{eqn:lln_c_bound}
\end{align}%
where $\xi \in \lbrack 0,N^{-1/2}c]$ and $(x)^+=\max\{x,0\}$.

We then deal with the second term in (\ref{eqn:2terms_c}). Let
\begin{equation*}
M_{\gamma }=\esssup_{t\geq 0}\mathbb{E}_{\mathbb{P}}\left[ \left\vert \gamma
^{\top }\phi \left( U,\mu \right) \right\vert ^{2+\epsilon _{0}}|\Phi (X)%
\mathcal{=}t\right] .
\end{equation*}%
For any $c\geq 0,$ we have
\begin{eqnarray*}
&&\sqrt{N}\int_{c/\sqrt{N}}^{+\infty }\mathbb{E}_{\mathbb{P}}\left[ \left.
\left( -\gamma ^{\top }\phi \left( U,\mu \right) +\sqrt{N}t\right)
^{-}\right\vert \Phi (X)=t\right] \mathrm{d}\mathbb{P}_{\Phi }(t) \\
&\geq &-\sqrt{N}\int_{c/\sqrt{N}}^{+\infty }\mathbb{E}_{\mathbb{P}}\left[
\gamma ^{\top }\phi \left( U,\mu \right) \mathbb{I}\{\left. \gamma ^{\top
}\phi \left( U,\mu \right) \geq \sqrt{N}t\right\vert \Phi (X)=t\right]
\mathrm{d}\mathbb{P}_{\Phi }(t) \\
&\geq &-\sqrt{N}\int_{c/\sqrt{N}}^{+\infty }\left( \frac{1}{\sqrt{N}t}%
\right) ^{1+\epsilon _{0}}\mathbb{E}_{\mathbb{P}}\left[ \left( \gamma ^{\top
}\phi \left( U,\mu \right) \right) ^{2+\epsilon _{0}}\mathbb{I}\{\left.
\gamma ^{\top }\phi \left( U,\mu \right) \geq \sqrt{N}t\right\vert \Phi (X)=t%
\right] \mathrm{d}\mathbb{P}_{\Phi }(t) \\
&\geq &-\left( \sqrt{N}\right) ^{-\epsilon _{0}}\int_{c/\sqrt{N}}^{+\infty }%
\frac{1}{t^{1+\epsilon _{0}}}\mathbb{E}_{\mathbb{P}}[\left\vert \gamma
^{\top }\phi \left( U,\mu \right) \right\vert ^{2+\epsilon _{0}}|\Phi (X)=t]%
\mathrm{d}\mathbb{P}_{\Phi }(t) \\
&\geq &-\left( \sqrt{N}\right) ^{-\epsilon _{0}}M_{\gamma }\int_{c/\sqrt{N}%
}^{+\infty }\frac{1}{t^{1+\epsilon _{0}}}\mathrm{d}\mathbb{P}_{\Phi }(t).
\end{eqnarray*}%
We pick $\varepsilon >0$ such that $\mathbb{P}_{\Phi }\left( \cdot \right) $
has density in $[0,\varepsilon ].$ Then, we have
\begin{eqnarray*}
&&\left( \sqrt{N}\right) ^{-\epsilon _{0}}M_{\gamma }\int_{c/\sqrt{N}%
}^{+\infty }\frac{1}{t^{1+\epsilon _{0}}}f(t)\mathrm{d}t \\
&=&M_{\gamma }\left( \sqrt{N}\right) ^{-\epsilon _{0}}\left(
\int_{\varepsilon }^{+\infty }\frac{1}{t^{1+\epsilon _{0}}}f(t)\mathrm{d}%
\mathbb{P}_{\Phi }(t)+\int_{c/\sqrt{N}}^{\varepsilon }\frac{1}{t^{1+\epsilon
_{0}}}f(t)\mathrm{d}t\right)  \\
&\leq &M_{\gamma }\left( \left( \sqrt{N}\right) ^{-\epsilon _{0}}\frac{1}{%
\varepsilon ^{1+\epsilon _{0}}}+\frac{1}{\epsilon _{0}}\left( \sqrt{N}%
\right) ^{-\epsilon _{0}}\left( \sqrt{N}/c\right) ^{\epsilon _{0}}f(\xi
)\right)  \\
&=&M_{\gamma }\left( \left( \sqrt{N}\right) ^{-\epsilon _{0}}\frac{1}{%
\varepsilon ^{1+\epsilon _{0}}}+\frac{1}{c^{\epsilon _{0}}\epsilon _{0}}%
f(\xi )\right) ,
\end{eqnarray*}%
where $\xi \in (c/\sqrt{N},\varepsilon ).\ $ By taking $\varepsilon
\downarrow 0,$ we have
\begin{equation*}
\liminf_{N\rightarrow +\infty }\sqrt{N}\int_{c/\sqrt{N}}^{+\infty }\mathbb{E}%
_{\mathbb{P}}\left[ \left. \left( -\gamma ^{\top }\phi \left( U,\mu \right) +%
\sqrt{N}t\right) ^{-}\right\vert \Phi (X)=t\right] \mathrm{d}\mathbb{P}%
_{\Phi }(t)\geq -\frac{M_{\gamma }f(0)}{c^{\epsilon _{0}}\epsilon _{0}}.
\end{equation*}%
Finally, by taking $c\uparrow +\infty ,$ we conclude step 2.

\textbf{Step 3: }We then apply weak law of triangular arrays \citet[Theorem
2.2.11]{ref:durrett2019probability}. We need to check
\begin{subequations}
\begin{eqnarray}
&&N \times \left[ \mathbb{P}\left( -\left( -\gamma ^{\top }\phi \left( U,\mathbb{\mu
}\right) +\sqrt{N}d(X)\right) \mathcal{C}(X)>\sqrt{N}\right) \right]
\rightarrow 0,\text{ and}  \label{eqn:step3_cond_1} \\
&&\mathbb{E}\left[ \left( \left( -\gamma ^{\top }\phi \left( U,\mu \right) +%
\sqrt{N}d(X)\right) ^{-}\right) ^{2}\mathcal{C}(X)\right] \rightarrow 0.
\label{eqn:step3_cond_2}
\end{eqnarray}%
For condition~\eqref{eqn:step3_cond_1}, we have
\end{subequations}
\begin{eqnarray*}
&&N\mathbb{P}\left( -\left( -\gamma ^{\top }\phi \left( U,\mu \right) +\sqrt{%
N}d(x_{i})\right) ^{-}\mathcal{C}(X)>\sqrt{N}\right)  \\
&\leq &N\mathbb{P}\left( \gamma ^{\top }\phi \left( U,\mu \right) \geq \sqrt{%
N}\right)  \\
&\leq &\frac{\mathbb{E}\left[ \left( \gamma ^{\top }\phi \left( U,\mathbb{%
\mu }\right) \right) ^{2+\epsilon _{0}}\right] }{\left( \sqrt{N}\right)
^{\epsilon _{0}}}\leq \frac{M_{\gamma }}{\left( \sqrt{N}\right) ^{\epsilon
_{0}}}\rightarrow 0.
\end{eqnarray*}

For condition~\eqref{eqn:step3_cond_2}, we have
\begin{eqnarray*}
&&\mathbb{E}\left[ \left( \left( -\gamma ^{\top }\phi \left( U,\mu \right) +%
\sqrt{N}d(X)\right) ^{-}\right) ^{2}\mathcal{C}(X)\right]  \\
&\leq &\mathbb{E}\left[ \left( \gamma ^{\top }\phi \left( U,\mu \right)
\right) ^{2}\mathbb{I\{}\gamma ^{\top }\phi \left( U,\mu \right) {\geq }%
\sqrt{N}d(X)\}\right]  \\
&=&\int_{0}^{+\infty }\mathbb{E}_{\mathbb{P}}\left[ \left( \gamma ^{\top
}\phi \left( U,\mu \right) \right) ^{2}\left. \mathbb{I}\left\{ \gamma
^{\top }\phi \left( U,\mu \right) \geq \sqrt{N}t\right\} \right\vert \Phi
(X)=t\right] \mathrm{d}\mathbb{P}_{\Phi }(t).
\end{eqnarray*}%
We pick $\varepsilon >0$ such that $\mathbb{P}_{\Phi }\left( \cdot \right) $
has density in $[0,\varepsilon ].$ Then, we have
\begin{eqnarray}
&&\int_{0}^{+\infty }\mathbb{E}_{\mathbb{P}}\left[ \left( \gamma ^{\top
}\phi \left( U,\mu \right) \right) ^{2}\left. \mathbb{I}\left\{ \gamma
^{\top }\phi \left( U,\mu \right) \geq \sqrt{N}t\right\} \right\vert \Phi
(X)=t\right] \mathrm{d}\mathbb{P}_{\Phi }(t)  \notag \\
&=&\int_{0}^{\varepsilon }\mathbb{E}_{\mathbb{P}}\left[ \left( \gamma ^{\top
}\phi \left( U,\mu \right) \right) ^{2}\left. \mathbb{I}\left\{ \gamma
^{\top }\phi \left( U,\mu \right) \geq \sqrt{N}t\right\} \right\vert \Phi
(X)=t\right] f(t)\mathrm{d}t  \label{eqn:step_3_first_term} \\
&&+\int_{\varepsilon }^{+\infty }\mathbb{E}_{\mathbb{P}}\left[ \left( \gamma
^{\top }\phi \left( U,\mu \right) \right) ^{2}\left. \mathbb{I}\left\{
\gamma ^{\top }\phi \left( U,\mu \right) \geq \sqrt{N}t\right\} \right\vert
\Phi (X)=t\right] \mathrm{d}\mathbb{P}_{\Phi }(t).
\label{eqn:step_3_second_term}
\end{eqnarray}%
For the first term (\ref{eqn:step_3_first_term}), we have
\begin{equation*}
\int_{0}^{\varepsilon }\mathbb{E}_{\mathbb{P}}\left[ \left( \gamma ^{\top
}\phi \left( U,\mu \right) \right) ^{2}\mathbb{I}\{\left. \gamma ^{\top
}\phi \left( U,\mu \right) \geq \sqrt{N}t\right\vert \Phi (X)=t\right] f(t)%
\mathrm{d}t\leq M_{\gamma }^{2/\left( 2+\epsilon _{0}\right) }\varepsilon
f(\xi ),
\end{equation*}%
where $\xi \in \lbrack 0,\varepsilon ].$ For the second term (\ref%
{eqn:step_3_second_term}) we have%
\begin{eqnarray*}
&&\int_{\varepsilon }^{+\infty }\mathbb{E}_{\mathbb{P}}\left[ \left( \gamma
^{\top }\phi \left( U,\mu \right) \right) ^{2}\mathbb{I}\{\left. \gamma
^{\top }\phi \left( U,\mu \right) \geq \sqrt{N}t\right\vert \Phi (X)=t\right]
\mathrm{d}\mathbb{P}_{\Phi }(t) \\
&\leq &\int_{\varepsilon }^{+\infty }\frac{\mathbb{E}_{\mathbb{P}}\left[
\left( \gamma ^{\top }\phi \left( U,\mu \right) \right) ^{2+\epsilon
_{0}}\Phi (X)=t\right] }{\left( \sqrt{N}t\right) ^{\epsilon _{0}}}\mathrm{d}%
\mathbb{P}_{\Phi }(t) \\
&\leq &\frac{M_{0}}{\left( \sqrt{N}\varepsilon \right) ^{\epsilon _{0}}}%
\rightarrow 0.
\end{eqnarray*}%
By taking $\varepsilon \downarrow 0,$ we have
\begin{equation*}
\int_{0}^{+\infty }\mathbb{E}_{\mathbb{P}}\left[ \left( \gamma ^{\top }\phi
\left( U,\mu \right) \right) ^{2}\mathbb{I}\{\left. \gamma ^{\top }\phi
\left( U,\mu \right) \geq \sqrt{N}t\right\vert \Phi (X)=t\right] \mathrm{d}%
\mathbb{P}_{\Phi }(t)\rightarrow 0.
\end{equation*}%
We then apply \citet[Theorem 2.2.11]{ref:durrett2019probability} to obtain
the weak law for each $\gamma .$

\textbf{Step 4: }We establish the Lipschitz continuity of%
\begin{equation*}
\sqrt{N}\mathbb{E}_{\mathbb{P}}\left[ \left( -\gamma ^{\top }\phi \left( U,%
\mathbb{\mu }\right) +\sqrt{N}d(X)\right) ^{-}\mc C(X)\right]
\end{equation*}
for $\left\Vert \gamma \right\Vert _{2}\leq B,$ which ensures the tightness.
For any $\gamma _{1},\gamma _{2}$ satisfying $\left\Vert \gamma
_{1}\right\Vert _{2}\leq B$ and $\left\Vert \gamma _{2}\right\Vert _{2}\leq
B,$ we have%
\begin{eqnarray*}
&&\sqrt{N}\left\vert \mathbb{E}_{\mathbb{P}}\left[ \left( -\gamma _{1}^{\top
}\phi \left( U,\mu\right) +\sqrt{N}d(X)\right) ^{-}\mc C(X)\right] \right. \\
&-& \left.\mathbb{E}_{\mathbb{P}}\left[ \left( -\gamma _{2}^{\top }\phi
\left( U,\mu\right) +\sqrt{N}d(X)\right) ^{-}\mc C(X) \}\right] \right\vert
\\
&\leq &\sqrt{N}\left\Vert \gamma _{1}-\gamma _{2}\right\Vert _{2}\left\Vert
\phi \left( U,\mu\right) \right\Vert _{2}\mc C(X)\mathbb{I}\left\{
B\left\Vert \phi \left( U,\mathbb{\mu }\right) \right\Vert _{2}\geq \sqrt{N}%
d(X)\right\} .
\end{eqnarray*}%
By following similar lines with steps 2 and 3, we have
\begin{eqnarray*}
\sqrt{N}\mathbb{E}_{\mathbb{\hat{P}}^{N}}\left[ \left\Vert \phi \left(
U,\mu\right) \right\Vert _{2}\mc C(X)\mathbb{I}\left\{B\left\Vert \phi
\left( U,\mu\right) \right\Vert _{2}\geq \sqrt{N}d(X)\right\} \right] 
\overset{p}{\longrightarrow }f(0)\mathbb{E}_{\mathbb{P}}\left[ \left.
\left\Vert \phi \left( U,\mu\right) \right\Vert _{2}^{2}\right\vert d(X%
\mathcal{)}=0\right] .
\end{eqnarray*}%
Then, by \citet[Theorem 7.5]{ref:billingsley2013convergence}, we have the
desired uniform convergence result. \end{proof}
\begin{proof}[Proof of Lemma \ref{lma:compactifaction}] Due to $\mathbb{E}\left[ \left.
\phi \left( U,\mu\right) \phi \left( U,\mu\right) ^{\top
}\right\vert d(X)=0 \right] \succ 0$, there
exists $\delta >0$ and $c_{0}\in (0,+\infty )$ such that
\begin{equation*}
\inf_{\left\Vert \gamma \right\Vert _{2}=1}\mathbb{E}\left[ \min \left\{
c_{0},\gamma ^{\top }\phi \left( U,\mu\right) \right\} \left\vert
\gamma ^{\top }\phi \left( U,\mu\right) \right\vert \right]
>\delta .
\end{equation*}%
for all $\left\Vert \gamma \right\Vert _{2}=1.$ And
\begin{equation*}
\inf_{\left\Vert \gamma \right\Vert _{2}=1}\mathbb{E}\left\vert \gamma
^{\top }\phi \left( U,\mu\right) \right\vert >0,
\end{equation*}%
since the unit circle is compact. Let $\delta =\inf_{\left\Vert \gamma
\right\Vert _{2}=1}\mathbb{E}\left\vert \gamma ^{\top }\phi \left( U,\mathbb{%
\mu }\right) \right\vert .$ For any $\varepsilon >0,$ there exists $N_{1}>0$
and $b^{\prime }<+\infty ,$ such that
\begin{equation*}
\mathbb{P}(\left\Vert V_{N}\right\Vert _{2}\geq b^{\prime })<\varepsilon /2,
\end{equation*}%
for any $N>N_{0}.$ Recalling Lemma \ref{lma:lln} and equation (\ref{eqn:lln_c_bound}), there exists $%
N_{0}>N_{1}$ such that
\begin{equation*}
\mathbb{P}\left( \exists \gamma :\left\Vert \gamma \right\Vert _{2}=b\text{
such that }M_{N}(\gamma )\geq -\frac{1}{4}\mathbb{E}\left[ \min \left\{
bc_{0},\left |\gamma ^{\top }\phi \left( U,\mu\right) \right| \right\} \left\vert
\gamma ^{\top }\phi \left( U,\mu\right) \right\vert \right]
\right) <\varepsilon /2
\end{equation*}%
for any $N>N_{0}$. Then, we have%
\begin{eqnarray*}
&&\inf_{\left\Vert \gamma \right\Vert _{2}=b}\mathbb{E}\left[ \min \left\{
bc_{0},\left |\gamma ^{\top }\phi \left( U,\mu\right) \right|  \right\} \left\vert
\gamma ^{\top }\phi \left( U,\mu\right) \right\vert \right]  \\
&\geq &b^{2}\inf_{\left\Vert \gamma \right\Vert _{2}=1}\mathbb{E}\left[ \min
\left\{ c_{0},\left |\gamma ^{\top }\phi \left( U,\mu\right) \right|  \right\}
\left\vert \gamma ^{\top }\phi \left( U,\mu\right) \right\vert %
\right] >b^{2}\delta .
\end{eqnarray*}%
Let $b=4b^{\prime }/\delta .$ We have
\begin{equation}
\mathbb{P}\left( \sup_{\left\Vert \gamma \right\Vert _{2}=b\text{ }%
}M_{N}(\gamma )\geq -bb^{\prime }\right) <\varepsilon /2.
\label{eqn:bound_b}
\end{equation}%
Notice that for any $\left\Vert \gamma \right\Vert _{2}>b,$%
\begin{equation}
M_{N}(\gamma )\leq \frac{\left\Vert \gamma \right\Vert _{2}}{b}M_{N}\left(
\frac{b}{\left\Vert \gamma \right\Vert _{2}}\gamma \right) \leq \frac{%
\left\Vert \gamma \right\Vert _{2}}{b}\sup_{\left\Vert \gamma \right\Vert
_{2}=b\text{ }}M_{N}(\gamma ).  \label{eqn:bound_gamma}
\end{equation}%
By combining inequalities (\ref{eqn:bound_b}) and (\ref{eqn:bound_gamma}),
we have
\begin{equation*}
\mathbb{P}\left( \exists \gamma :\left\Vert \gamma \right\Vert _{2}>b,\text{
such that }M_{N}(\gamma )\geq -\left\Vert \gamma \right\Vert _{2}b^{\prime
}\right) <\varepsilon /2.
\end{equation*}%
Therefore,
\begin{eqnarray*}
&&\mathbb{P}\left( \sup_{\left\Vert \gamma \right\Vert _{2}>b}\left\{ \gamma
^{\top }V_{N}+M_{N}(\gamma )\right\} >0\right)  \\
&\leq &\mathbb{P}\left( \sup_{\left\Vert \gamma \right\Vert _{2}>b}\left\{
\left\Vert \gamma \right\Vert _{2}\left\Vert V_{N}\right\Vert
_{2}+M_{N}(\gamma )\right\} >0\right)  \\
&\leq &\mathbb{P}(\left\Vert V_{N}\right\Vert _{2}\geq b^{\prime })+\mathbb{P}%
\left( \exists \gamma :\left\Vert \gamma \right\Vert _{2}>b,\text{ such that
}M_{N}(\gamma )\geq -\left\Vert \gamma \right\Vert _{2}b^{\prime }\right)  \\
&\leq &\varepsilon .
\end{eqnarray*}
This completes the proof.
\end{proof}

\section{Additional Details for Numerical Experiments}
\subsection{Validation of the Hypothesis Test}
\label{sec:numerical_detail}

In this section, we empirically validate the convergence result in Theorem \ref%
{thm:clt} and our proposed hypothesis test method.
we use a simple logistic classifier in the form
\begin{equation*}
\mc C(x)=\mathbb{I}\left\{ \frac{1}{1+\exp \left( -\theta ^{\top }x\right) }\geq
\tau \right\} .
\end{equation*}%
Then, the decision boundary is
$
\left\{ x:\theta ^{\top }x=-\log \left( \frac{1}{\tau }-1\right)
\right\} .
$
We denote $w=-\log \left( \frac{1}{\tau }-1\right) .$ Then, we borrowed the
example in \citet{ref:taskesen2020statistical}. Let

\begin{equation*}
p_{11}=0.4,p_{01}=0.1,p_{10}=0.4,p_{00}=0.1.
\end{equation*}

Moreover, conditioning on $(A,Y)$, the feature $X$ follows a Gaussian
distribution of the form%
\begin{align*}
X|A=1,Y=1\sim & \mathcal{N}([6,0],[3.5,0;0,5]), \\
X|A=0,Y=1\sim & \mathcal{N}([-2,0],[5,0;0,5]), \\
X|A=1,Y=0\sim & \mathcal{N}([6,0],[3.5,0;0,5]), \\
X|A=0,Y=0\sim & \mathcal{N}([-4,0],[5,0;0,5]).
\end{align*}%
The true distribution $\mathbb{P}$ is thus a mixture of Gaussian. A simple
algebraic calculation indicates that a logistic classifier with $\theta
=(0,1)^{\top }$ and $\tau=0.5$ is fair with respect to the equal opportunity criterion in
Example \ref{eg:EOpp}. Let $\varphi (\cdot )$ denotes the density of the
standard normal distribution and we denote $\mu _{ay}$ and $\Sigma _{ay}$ to
be the conditional mean and variable defined above, respectively. For any $%
\theta ,$ the density of $\theta ^{\top }X$ becomes
\begin{equation*}
\sum_{a,y\in \left\{ 0,1\right\} ^{2}}\left( \theta ^{\top }\Sigma
_{ay}\theta \right) ^{-1/2}p_{ay}\varphi \left( \left( \theta ^{\top }\Sigma
_{ay}\theta \right) ^{-1/2}\left( \theta ^{\top }x-\theta ^{\top }\mu
_{ay}\right) \right) .
\end{equation*}%
And thus the density of $\Phi \left( \cdot \right) $ becomes
\begin{equation*}
f\left( z\right) =\left\Vert \theta \right\Vert _{\ast }\sum_{a,y\in \left\{
0,1\right\} ^{2}}\left( \theta ^{\top }\Sigma _{ay}\theta \right)
^{-1/2}p_{ay}\varphi \left( \left( \theta ^{\top }\Sigma _{ay}\theta \right)
^{-1/2}\left( (z\left\Vert \theta \right\Vert _{\ast }+w)-\theta ^{\top }\mu
_{ay}\right) \right) .
\end{equation*}%
By Bayes' formula, we have
\begin{align*}
p_{ay|d(X)=0}=f\left( 0\right) ^{-1}\left( \theta ^{\top }\Sigma
_{ay}\theta \right) ^{-1/2}\left\Vert \theta \right\Vert _{\ast
}p_{ay} \varphi \left( \left( \theta ^{\top }\Sigma _{ay}\theta \right)
^{-1/2}\left( w-\theta ^{\top }\mu _{ay}\right) \right)
\end{align*}
for $a\in
\left\{ 0,1\right\}$ and $y\in \left\{ 0,1\right\}$, where $p_{ay|d(X)=0} = \mathbb{E}[\mathbb{I}_{(a,y)}(A,Y)|d(X)=0]$.
In the first experiments, we generate $N \in \{30,100,500\}$ i.i.d.~samples from $\mathbb{P}$ and then calculate $N\times \mathcal{D}(\hat{\mathbb{P}}^N)$. We replicate this process for 2,000 times and compare the empirical distribution of $N\times \mathcal{D}(\hat{\mathbb{P}}^N)$ with the limiting distribution defined in Theorem \ref{thm:clt}. Figure \ref{fig:histogram} shows that finite-sample empirical estimates are closed to the theoretical limiting distributions even when $N$ is as small as $30$.

\begin{figure}[!ht]
\centering
\subfigure[$N = 30$]{
\label{pic:N_30} \includegraphics[width=2.01in]{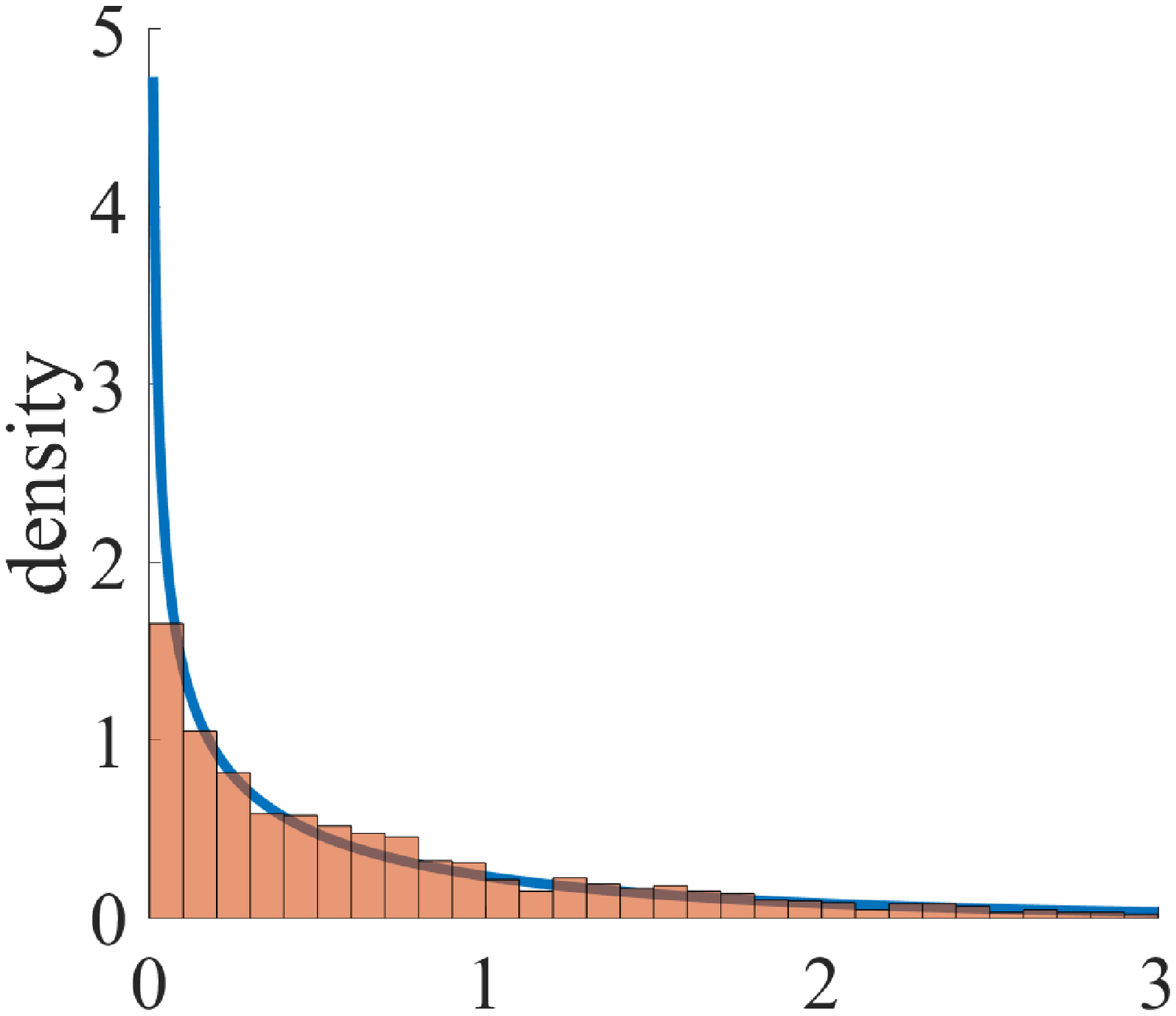}}
\subfigure[$N = 100$]{
\label{pic:N_100} \includegraphics[width=2.01in]{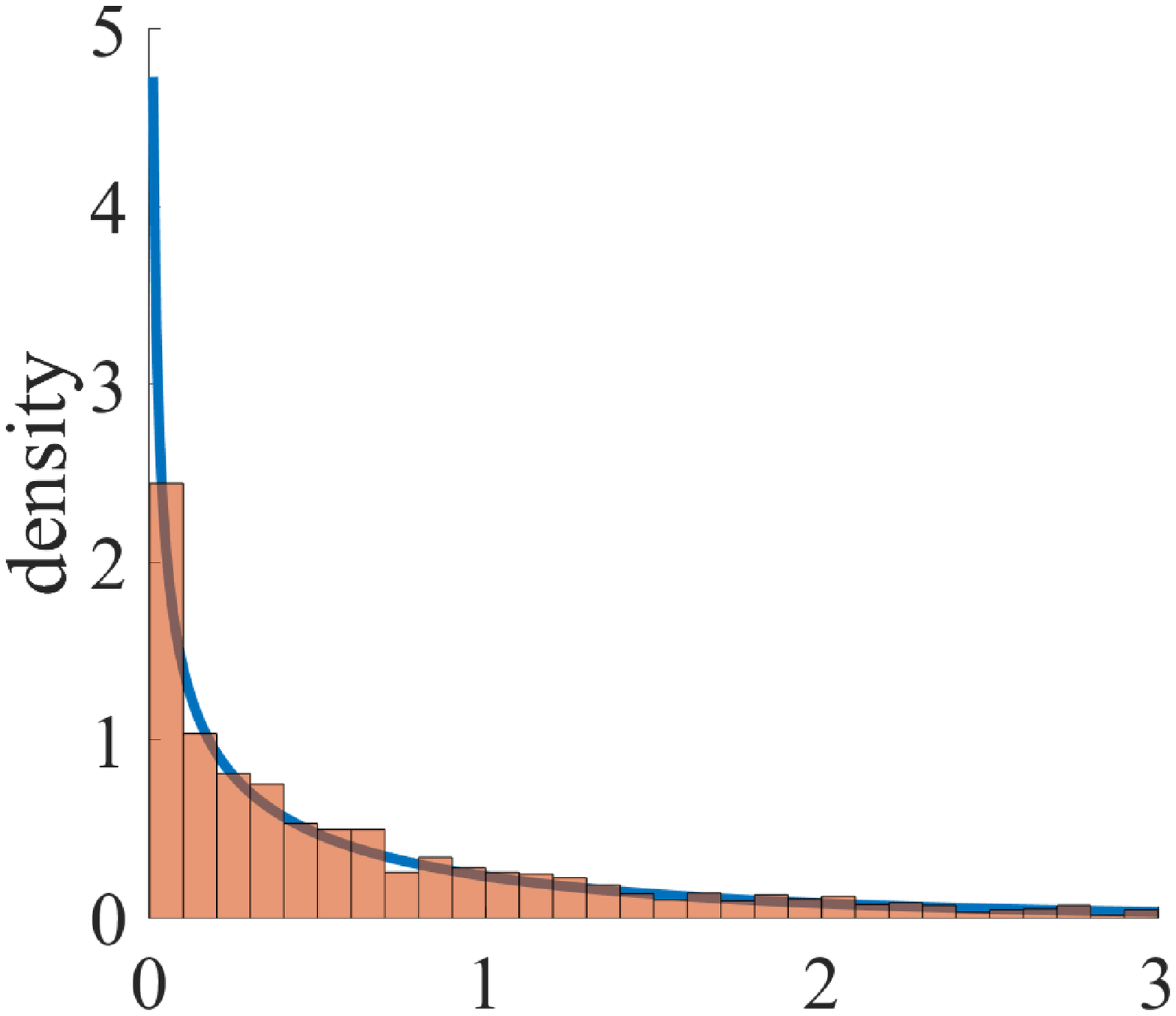}}
\subfigure[$N = 500$]{
\label{pic:N_500} \includegraphics[width=2.01in]{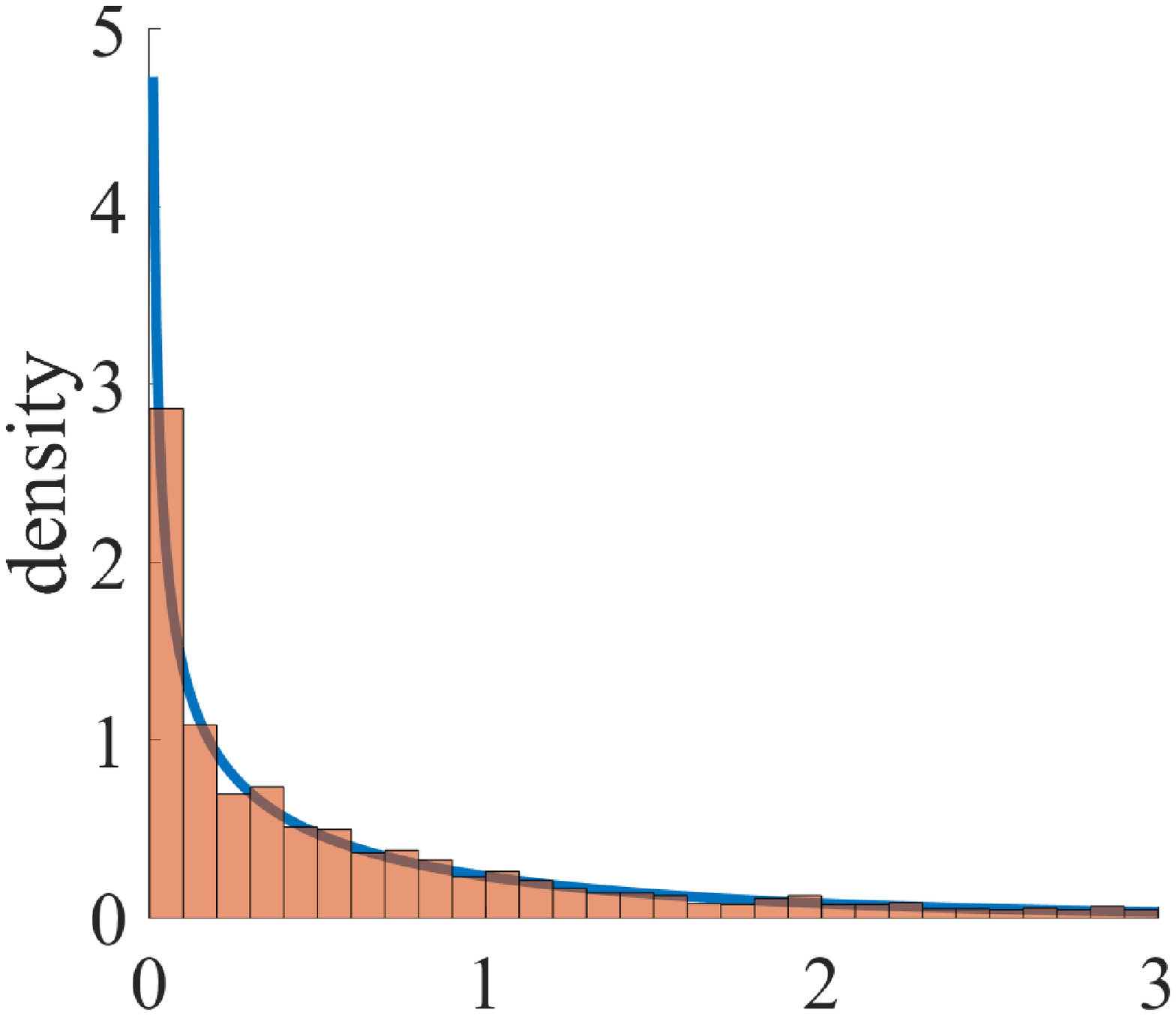}}
\vspace{-3mm}
\caption{Empirical distribution $N\times \mathcal{D}(\hat{\mathbb{P}}^N)$ over 2,000 replications (histogram) versus the limiting Chi-square
distribution (blue curve) with different sample sizes $N$.}
\label{fig:histogram}
\end{figure}

In the second experiments, we show that our proposed Wasserstein projection hypothesis test has the desired coverage property. We generate $N \in \{30,100,500,1000,2000\}$ i.i.d.~samples from $\mathbb{P}$ and compute the estimate $\hat{S}$ defined in Section~\ref{sec:kernel_estimate} and  the empirical covariance using the sample data. For the kernel estimator $\hat{S}$, we use the standard Gaussian kernel and choose the bandwidth $h = N^{-1/5}$, where the results listed below are not sensitive to the constant. We repeat the procedure for 2,000 replications and report the rejection probability at different significant values of $\alpha \in \{0.1,0.05,0.01\}$ in Table \ref{tab:rej_prob}. We can observe that when $N>100$, the rejection probability is closed to the desired level $\alpha$.
\begin{table}[htbp]
\caption{Comparison of the null rejection probabilities of probabilistic
equal opportunity tests with different significance levels $\protect\alpha$
and test sample sizes $N$.}
\label{tab:rej_prob}\centering
\begin{tabular}{lccc}
\toprule $\alpha$ & 0.10 & 0.05 & 0.01 \\
\midrule $N=30$ & 0.2875 &   0.2255  &  0.1415 \\
$N=100$ & 0.0945 &   0.0540  &  0.0250 \\
$N=500$ & 0.0895  &  0.0450  &  0.0085 \\
$N=1000$ & 0.0900  &  0.0430   & 0.0065 \\
$N=2000$ & 0.0870 &   0.0460 &   0.0080 \\
\bottomrule
\end{tabular}%
\end{table}
\subsection{The Description of Datasets}
Followings show brief descriptions of datasets: Arrhythmia, COMPAS and Drug~\cite{fehrman2017five} provided in Section~\ref{sec:real}.
\begin{itemize}
\item \textbf{Arrhythmia} is from UCI repository\footnote{\url{https://archive.ics.uci.edu/ml/datasets/arrhythmia}}, where the aim of this data set is to  distinguish between the presence and absence of cardiac arrhythmia and classify it in one of the 16 groups. The dataset consists of 452 samples and we use the first 12 features among which the gender is the sensitive feature. For our purpose, we construct binary labels between `class 01' (`normal') and all other classes (different classes of arrhythmia and unclassified ones).
\item \textbf{COMPAS} (Correctional Offender Management Profiling for Alternative Sanctions)\footnote{\url{https://www.propublica.org/datastore/dataset/compas-recidivism-risk-score-data-and-analysis}} is a commerical tool used by judges, probation and parole officers to estimate a criminal defendant's likelihood to re-offend algorithmically. The COMPAS dataset contains the criminal records within 2 years after the decision. We use race (African-American and Caucasian, which accounts for 5278 samples) as the sensitive attribute.
\item \textbf{Drug}~\cite{fehrman2017five} contains  answers of 1885 participants on their use of 17 legal and illegal drugs. We concern the cannabis usage as a binary problem, where the label is `Never used' VS `Others' (`used'). There are 12 features including age, gender, education, country, ethnicity, NEO-FFI-R measurements, impulsiveness measured by BIS-11 and  sensation seeing measured by ImpSS. Among those, we choose ethnicity (black vs others) as the sensitive attribute.
\end{itemize}
\label{sec:dataset}

%%%%%%%%%%%%%%%%%%%%%%%%%%%%%%%

\bibliographystyle{icml2021}
\bibliography{arxiv_main.bbl}

%%%%%%%%%%%%%%%%%

\end{document}